\documentclass[11pt]{article}

\usepackage[preprint]{acl}

\usepackage{times}
\usepackage{latexsym}

\usepackage[T1]{fontenc}
\usepackage[utf8]{inputenc}

\usepackage{microtype}

\usepackage{inconsolata}

\usepackage{amsmath}
\usepackage{amssymb}
\usepackage{amsthm}
\usepackage{graphicx}
\usepackage{amsfonts}
\usepackage[linesnumbered,ruled,vlined]{algorithm2e}
\usepackage{enumitem}
\usepackage{tabularx}
\usepackage{booktabs}
\usepackage{multirow}
\usepackage[table]{xcolor}
\usepackage{hyperref}
\usepackage{cleveref}
\usepackage{url}
\usepackage{tcolorbox}
\usepackage{pifont}

\newtheorem{theorem}{Theorem}
\newtheorem{lemma}{Lemma}
\newtheorem{definition}{Definition}

\newcommand{\system}{SHIP\xspace} % definition of SHIP

\definecolor{ForestGreen}{HTML}{188038}
\definecolor{BrickRed}{HTML}{D93025}
\newcommand{\cmark}{\textcolor{ForestGreen}{\ding{51}}}
\newcommand{\xmark}{\textcolor{BrickRed}{\ding{55}}}

\title{SHIP: Structured and Hierarchical Insight with Privacy}

\author{
  \textbf{Kai Chen}\textsuperscript{1}\thanks{This work was primarily done during an internship at Google.},
  \textbf{Guisi Ni}\textsuperscript{2},
  \textbf{Albert Cheu}\textsuperscript{2},
  \textbf{Kassem Fawaz}\textsuperscript{2}, \\
  \textbf{Yuzheng Hu}\textsuperscript{2},
  \textbf{Ryan McKenna}\textsuperscript{2},
  \textbf{Peter Kairouz}\textsuperscript{2} \\
  \textsuperscript{1}University of Virginia \quad
  \textsuperscript{2}Google Research \\
  \texttt{kaichen@virginia.edu} \\
  \texttt{\{guisin,cheu,kfawaz,yuzhenghu,mckennar,kairouz\}@google.com}
}

\begin{document}
\maketitle
\begin{abstract}
Extracting insights from textual data while protecting user privacy is increasingly important. Existing methods produce private insights in diverse forms for different analytical purposes, making their utility difficult to assess and compare. In this work, we formulate differentially private (DP) insights as reusable representations for answering analytical queries, providing a common basis for utility assessment. Answering these queries, however, requires fine-grained details and cross-attribute statistics that existing representations may omit. To capture this information, we introduce Structured and Hierarchical Insight with Privacy (\system), which automatically maps queries to attributes, builds a hierarchical hybrid taxonomy, and releases a DP synthetic table for further analysis. We also develop an evaluation protocol based primarily on QA performance and structured fidelity. Our experiments demonstrate \system's strong performance on both measures relative to existing private insight baselines, while identifying queries that remain challenging for structured representations.
\end{abstract}

% main paper body 

\section{Introduction} 

Modern organizations collect vast amounts of textual data from their users. In particular, the rapid development of generative AI has introduced new user-facing tools, such as chatbots and AI agents, that continuously generate large volumes of data. Analyzing these data is crucial for organizations to understand user needs, improve their products and models, and guide decision-making. However, these data often contain personal information, making privacy protection essential for such analysis.

Various insight extraction systems have been developed to analyze such textual data. For example, Anthropic proposes Clio \citep{tamkin2024clioprivacypreservinginsightsrealworld}, which transforms large collections of AI-assistant conversations into extracted facets, hierarchical semantic clusters, and cluster-level summaries. Google's Urania \citep{liu2025uraniadifferentiallyprivateinsights} extends this paradigm with formal end-to-end DP guarantees, offering stronger privacy protection against adversarial attacks \citep{annamalai2026cliopatra}. Moreover, OpenAI \citep{chatterji2025chatgpt} and Google \citep{cheu2025provablyprivateanalyticsinsights} have explored structured approaches to extracting insights. Google's approach additionally provides formal DP guarantees through DP aggregation, with LLM-based structured summarization in trusted execution environments. Several other works further investigate applications of insights across various domains such as healthcare, science, and generative AI usage analysis \citep{codreanu2026aiinscience,anthropic2025affective,anthropic2026healthcare,costagomes2026copilothealth}.

These systems produce insights in different forms, reflecting different analytical purposes. 
Clio and Urania support exploratory analysis with limited prior specification, while structured approaches organize information along targeted attributes.
Existing evaluations are often tailored to individual methods' purposes and output formats, leaving no common basis for comparing analytical utility across methods \citep{tamkin2024clioprivacypreservinginsightsrealworld,liu2025uraniadifferentiallyprivateinsights}. We therefore need a shared description of information needs that accommodates diverse analytical purposes and provides concrete objectives for DP insight design and evaluation.

% For example, Clio and Urania are designed for exploratory analysis without predefined schemas or query workloads, but they rely on manual facet engineering and unsupervised clustering in embedding space, which is difficult to tune under DP and entangles distinct semantic dimensions into coarse cluster summaries. Conversely, structured approaches map records directly to explicit categorical attributes, but existing methods rely on flat, single-level histograms that either miss fine-grained subcategories under DP noise or require static, manually specified taxonomies.

In this work, we focus on a setting in which analysts have a \emph{workload} of dataset-level analytical queries to guide insight extraction. The workload makes information needs explicit and provides a concrete criterion for utility. For example\footnote{Refer to Table 4 and Table 5 for more examples.}, access to a medical dataset could inspire the question ``What are the users' primary symptoms and chief complaints, and for those symptoms, what diagnostic tests do the doctors suggest?'' Answering such queries requires both \emph{insight depth} (resolving fine-grained categories within each attribute, such as specific symptoms and tests) and \emph{insight breadth} (preserving cross-attribute joint distributions, such as which tests are suggested for which symptoms). In this setting, cluster-level summaries and keyword statistics are often too coarse for detailed attribute statistics and do not directly expose cross-attribute joint distributions. Flat structured baselines map records directly to explicit categorical attributes, but rely on flat, single-level taxonomies that either miss fine-grained subcategories or require static schemas. These limitations motivate a workload-driven DP framework that simultaneously captures insight depth and insight breadth.

To address these challenges, we formalize the DP insight extraction, develop a workload-driven extraction method, and propose a shared evaluation protocol. Our main contributions are as follows:

\begin{enumerate}[leftmargin=*, label=\textbullet]
\item \textbf{Workload-based Formulation and Evaluation.} We formulate DP insight extraction from a query-answering perspective, viewing an insight as a reusable representation of a textual corpus that supports a workload of analytical queries. To evaluate analytical utility independently of the output format, we use a shared offline agent to answer the same queries solely from each method's released insights, and introduce a complementary metric to evaluate semantic and distributional fidelity for structured insights.

\item \textbf{The \system Framework.}
We introduce \system, an end-to-end framework for extracting structured insights under rigorous DP guarantees. It automatically maps the workload to attributes, builds a hierarchical and hybrid taxonomy that combines private open-set discovery with adaptive supplementation by a taxonomy model \citep{davidson2026reasoningdrivensyntheticdatageneration} to retain fine-grained details (depth), and applies DP tabular synthesis \citep{mckenna2024aimadaptiveiterativemechanism} to release a reusable private table that preserves cross-attribute correlations (breadth).

\item \textbf{Empirical Analysis of Utility and Limitations.}
With its workload-aware tabular structure and hierarchical taxonomy, \system achieves average gains of 12.7\% points in deterministic QA accuracy and 9.0 points in open-ended QA $F_1$ on WildChat, compared with the strongest baseline across privacy budgets. We also identify query structures, such as long topic-switching sequences, that remain challenging for \system's current representation.
\end{enumerate}

% \ryan{I think we should show a diagram of the high-level capability / vision of some kind here.  Can maybe shorten algorithms or move to the appendix to free up space.}
% \ryan{worth looking into what the types of work / contributions the venue values, and framing the intro around that to some extent. }

\section{Preliminaries}

\subsection{Differential Privacy}

To provide formal privacy guarantees over sensitive user interaction data, we ground our framework in Differential Privacy (DP). Two datasets $\mathcal{D}, \mathcal{D}'$ are defined as \textit{neighboring} (denoted $\mathcal{D} \sim \mathcal{D}'$) if they differ by at most one individual record.

\begin{definition}[Differential Privacy \citep{dwork2006calibrating}]
\label{def:dp}
A randomized mechanism $\mathcal{M}$ satisfies $(\varepsilon, \delta)$-differential privacy if, for all neighboring datasets $\mathcal{D} \sim \mathcal{D}'$ and every measurable subset of possible outputs $\mathcal{S} \subseteq \text{Range}(\mathcal{M})$:
\begin{equation}
    \mathbb{P}[\mathcal{M}(\mathcal{D}) \in \mathcal{S}] \le e^\varepsilon \, \mathbb{P}[\mathcal{M}(\mathcal{D}') \in \mathcal{S}] + \delta.
\end{equation}
%When $\delta = 0$, $\mathcal{M}$ satisfies pure $\varepsilon$-DP.
\end{definition}

To simplify privacy accounting across iterative operations, we track the privacy loss using zero-Concentrated Differential Privacy (zCDP).

\begin{definition}[$(\delta, \rho)$-zCDP \citep{bun2016concentrateddifferentialprivacysimplifications}]
\label{def:zcdp}
A randomized mechanism $\mathcal{M}$ satisfies $\delta$-approximate $\rho$-zCDP if, for every pair of neighboring datasets $\mathcal{D}\sim\mathcal{D}'$, there exist events $E$ and $E'$ with probabilities at least $1-\delta$ such that, letting $P$ and $Q$ denote the output distributions of $\mathcal{M}(\mathcal{D})$ and $\mathcal{M}(\mathcal{D}')$ conditioned on $E$ and $E'$, respectively, for every $\alpha>1$,
\begin{equation}
\max\{D_\alpha(P\Vert Q),D_\alpha(Q\Vert P)\}\le\rho\alpha,
\end{equation}
where $D_\alpha(P\Vert Q)$ denotes the R\'enyi divergence of order $\alpha$. When $\delta=0$, this reduces to $\rho$-zCDP.
\end{definition}

\subsection{Terminology}

In this section, we introduce the terminology used throughout the paper. Let $\mathcal{D} = \{x_1, \dots, x_N\}$ denote a corpus of $N$ textual records, such as user prompts or multi-turn dialogues. We begin by describing analysts' information needs through queries and workloads, and then define insights as reusable resources that support these needs.

\begin{definition}[Query and Workload]
\label{def:query_workload}
A \textbf{query} $q$ is a natural-language specification of a dataset-level analytical information need, asking about properties, patterns, or relationships in a dataset $\mathcal{D}$. A \textbf{workload} $\mathcal{W}=\{q_1,\ldots,q_k\}$ is a collection of queries representing the analytical needs.
\end{definition}

Throughout this work, we focus on dataset-level analytical queries rather than retrieving information from individual records, as the latter are inherently less compatible with privacy. 

\begin{definition}[Workload-Aware Insight]
\label{def:insight}
Given a dataset $\mathcal{D}$ and a workload $\mathcal{W}$, an \textbf{insight} $\mathcal{S}$ is a data structure derived from $\mathcal{D}$ that can provide answers to the queries in $\mathcal{W}$.
\end{definition}

In this work, we primarily study \textit{structured insights}, which organize information in a tabular form. We therefore introduce the basic components of this representation.

\begin{definition}[Attribute \& Taxonomy]
\label{def:taxonomy}
An \textbf{attribute} $A_i$ is an observable semantic dimension, and a \textbf{taxonomy} $\mathcal{T} = \{A_1:\Omega_1, \dots, A_m:\Omega_m\}$ associates each attribute $A_i$ with its value domain $\Omega_i$.
\end{definition}

\subsection{Problem Statement}

With the necessary definitions provided in the previous sections, we define the workload-driven DP insight extraction problem as follows. 
\begin{definition}[Workload-Driven DP Insight Extraction]
\label{def gen}
Let $\mathcal{D}$ be a private textual dataset and $\mathcal{W}$ a workload of analytical queries. The task of workload-driven DP insight extraction is to design a mechanism $\mathcal{F}$ that releases an insight representation $\mathcal{S}=\mathcal{F}(\mathcal{D},\mathcal{W})$ under $(\varepsilon,\delta)$-DP, such that $\mathcal{S}$ serves as an interpretable and reusable resource from which queries in $\mathcal{W}$ can be answered without further access to $\mathcal{D}$.
\end{definition}

% \begin{theorem}[zCDP to DP \citep{bun2016concentrateddifferentialprivacysimplifications}]
% \label{thm:zcdp_to_dp}
% If a randomized mechanism $\mathcal{M}$ satisfies $(\delta', \rho)$-zCDP, then for any target $\delta > \delta'$, $\mathcal{M}$ satisfies $(\varepsilon, \delta)$-differential privacy, where:
% \begin{equation}
%     \varepsilon = \rho + 2\sqrt{\rho \ln\left(\frac{1}{\delta - \delta'}\right)}.
% \end{equation}
% \end{theorem}

\section{Related Work}
\label{sec:existing_paradigms}

Existing approaches to private insight follow a common two-stage structure: a language model inference step that extracts annotations or summaries from each record, followed by an aggregation and privatization step (via clustering, aggregation, and/or differential privacy) that produces a dataset-level release. Within this shared framework, existing techniques can be organized along five capabilities (\Cref{tab:paradigm_taxonomy}): whether they satisfy formal differential privacy, capture \emph{insight depth} (multi-level hierarchical granularity within a semantic dimension), capture \emph{insight breadth} (joint distributions across multiple semantic dimensions), discover categories from data (open-set), and adapt to an explicit query workload.

\noindent \textbf{Clustering-based approaches} include Clio \citep{tamkin2024clioprivacypreservinginsightsrealworld} and its differentially private counterpart Urania \citep{liu2025uraniadifferentiallyprivateinsights}. These methods prompt a language model to extract a list of keywords from each record $x_i \in \mathcal{D}$, embed the keywords into a high-dimensional vector space, and partition the records into $K$ clusters. Clio relies on empirical privacy filters rather than formal DP guarantees, making it vulnerable to reconstruction attacks \citep{annamalai2026cliopatra}. Urania instead selects keywords through DP partition selection and clusters data via DP $k$-means. While effective for unconstrained exploratory analysis, assigning different records to a single cluster entangles multiple semantic dimensions and does not expose pairwise or higher-order distributions across attributes. 
% Moreover, combining high-dimensional private clustering with per-cluster keyword thresholding is challenging because both operations incur substantial noise when $\varepsilon$ is small. 
% \ryan{can we cite some paper that says DP clustering is a hard problem  with poor privacy / utility trade-offs?}

\begin{table}[t]
\centering
\caption{Taxonomy of insight extraction paradigms across five core capabilities: Clio \citep{tamkin2024clioprivacypreservinginsightsrealworld}, Urania \citep{liu2025uraniadifferentiallyprivateinsights}, \citet{chatterji2025chatgpt}, \texttt{Flat-Open} \citep{cheu2025provablyprivateanalyticsinsights}, \texttt{Flat-Close} \citep{hu2025actg}, and SHIP.}
\label{tab:paradigm_taxonomy}
\footnotesize
\setlength{\tabcolsep}{2.5pt}
\resizebox{\columnwidth}{!}{%
\begin{tabular}{@{}lccccc@{}}
\toprule
\textbf{Method} & \textbf{DP} & \textbf{Depth} & \textbf{Breadth} & \textbf{Open-Set} & \textbf{Workload} \\
\midrule
Clio & \xmark & \cmark & \xmark & \cmark & \xmark \\
Urania & \cmark & \cmark & \xmark & \cmark & \xmark \\
Chatterji et al. & \xmark & \xmark & \cmark & \xmark & \xmark \\
\texttt{Flat-Open} & \cmark & \xmark & \xmark & \cmark & \xmark \\
\texttt{Flat-Close} & \cmark & \xmark & \cmark & \xmark & \xmark \\
\midrule
\textbf{SHIP (Ours)} & \cmark & \cmark & \cmark & \cmark & \cmark \\
\bottomrule
\end{tabular}%
}
\end{table}

\noindent \textbf{Taxonomy-based approaches} map each record directly to discrete categories along explicit semantic attributes \citep{chatterji2025chatgpt, cheu2025provablyprivateanalyticsinsights, hu2025actg}. Given a collection of $m$ attributes, structured approaches use a language model to classify each record $x_i \in \mathcal{D}$ into a tuple of categories $(y_i^{(1)}, \dots, y_i^{(m)}) \in \Omega_1 \times \dots \times \Omega_m$, reducing the unstructured corpus to a simple table. However, existing mechanisms have limitations in insight depth and insight breadth. On the one hand, open-set structured summarization (\texttt{Flat-Open}; \citealp{cheu2025provablyprivateanalyticsinsights}) discovers categories through DP partition selection \citep{desfontaines2021differentiallyprivatepartitionselection}, which can miss low-frequency categories. The single-attribute formulation considered in their work does not explicitly address multi-attribute analysis. Extending it through independent $1$-way histograms would divide the privacy budget across attributes, raising selection thresholds and further reducing category coverage. These independent releases would still not preserve cross-attribute joint distributions.

On the other hand, closed-set structured synthesis (\texttt{Flat-Close}; \citealp{chatterji2025chatgpt, hu2025actg}) specifies each candidate domain $\Omega_j$ a priori using domain experts or a zero-shot language model without accessing the private dataset $\mathcal{D}$. After classifying each record into the pre-specified taxonomy, the differentially private variant \citep{hu2025actg} applies a marginal-based synthetic data mechanism such as AIM \citep{mckenna2024aimadaptiveiterativemechanism}, which privately measures low-order marginals and synthesizes a tabular dataset $\tilde{\mathcal{D}}$. While this approach captures insight breadth, its utility depends on the quality of the taxonomy. Finally, none of these existing paradigms formalize the downstream analytical queries that the insight is intended to support, leaving attribute selection manual and cross-format evaluation ad hoc.

\section{Method}
\label{sec:method}

In this section, we introduce our method, Structured Hierarchical Insights with Privacy (SHIP), which works in three steps: it identifies the information dimensions required by the workload, constructs a hierarchical structured representation over these dimensions, and privately synthesizes their joint distribution.

\begin{algorithm}[!t]
\caption{Hierarchical Hybrid DP Taxonomy Building}
\label{alg:hybrid discovery}
\KwIn{Records $\mathcal{D}=\{x_1,\dots,x_N\}$, Dataset description $S_D$, Max depth $d$, Attribute $A$, Privacy parameters $(\rho_{\text{step}},\delta_{\text{step}})$, Taxonomy model $\mathcal{M}_{\text{tax}}$, Classification model $\mathcal{M}_{\text{cls}}$.}
\KwOut{Taxonomy $\mathcal{T}_A$, Internal record paths $\{y_i\}_{i=1}^N$.}

$\rho_{\text{open}}=0.9\rho_{\text{step}},\quad \rho_{\text{close}}=0.1\rho_{\text{step}}$\;
$\sigma_{\text{open}}=1/\sqrt{2\rho_{\text{open}}},\quad \sigma_{\text{close}}=1/\sqrt{2\rho_{\text{close}}}$\;
Compute $\tau_{\text{open}}$ and $\tau_{\text{close}}$\;
Initialize tree-like domain $\Omega_A$ with root $\bot$\;

\For{$l \leftarrow 1$ \KwTo $d$, each level $l-1$ category $c_p$ in $\Omega_A$}{
    % Skip if $c_p$ is \texttt{other} or \texttt{none}\;
    $\mathcal{D}_p \leftarrow$ records assigned to $c_p$\;

    Classify $\mathcal{D}$ via $\mathcal{M}_\text{cls}$ to obtain domain $\mathcal{C}_{\text{open}}$ and allocation results $\{y_i\}$ and conduct partition selection\;

    Obtain supplementary domain $\mathcal{C}_{\text{comp}}$ through $\mathcal{M}_{\text{tax}}$ and re-allocation results $\{y_i\}$ through $\mathcal{M}_{\text{cls}}$\;

    $\mathcal{C}_{\text{surv}} = \mathcal{C}_{\text{open}} \cup \mathcal{C}_{\text{comp}}$
    
    Add $\mathcal{C}_{\text{surv}}$ as children of $c_p$ in $\Omega_A$ and allocate $\mathcal{D}_p$ using $\{y_i\}$\;
}
\Return{$\{A:\Omega_A\},\{y_i\}_{i=1}^N$}\;
\end{algorithm}

\subsection{DP Taxonomy Building}
\label{subsec:taxonomy}

As discussed in \Cref{sec:existing_paradigms}, open-set discovery may miss low-frequency categories that fall below the DP threshold, whereas closed-set taxonomies rely on static priors that may be inaccurate (see \Cref{{sec:main exp}} and \Cref{app:cat proportion}). To capture fine-grained \emph{insight depth} while maintaining coverage, we introduce a hierarchical hybrid category proposal mechanism (\Cref{alg:hybrid discovery}, \Cref{fig:ship}; full pseudocode in \Cref{alg:hybrid discovery full}) that combines open-set discovery with closed-set supplementation across multiple levels of granularity \citep{bagdasaryan2022sparsefederatedanalyticslocation, zhu2020federatedheavyhittersdiscovery, tamkin2024clioprivacypreservinginsightsrealworld}.

\begin{figure}[!t]
\centering
\includegraphics[width=\columnwidth]{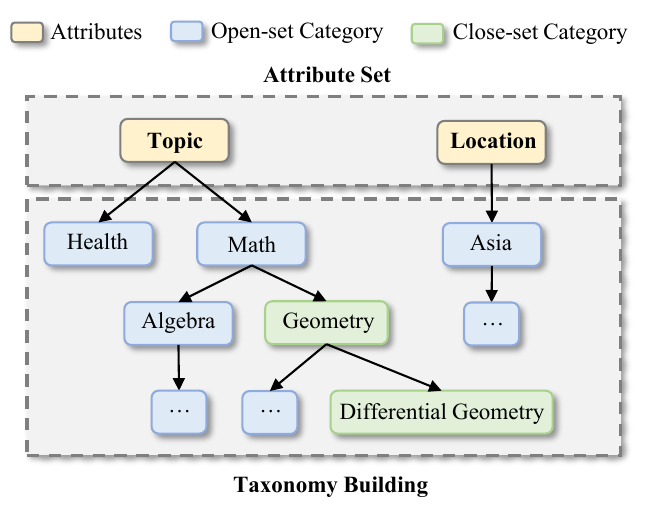}
\caption{An illustration of \system's taxonomy building step. The hierarchical taxonomy originates from the top-level attribute set, integrating with open-set categories (blue) with complementary close-set categories (green) for multi-level categorization.}
\label{fig:ship}
\end{figure}

Taxonomy construction proceeds top-down for up to $d$ levels. At each parent category $c_p$, we first use the classification model $\mathcal{M}_{\text{cls}}$ to freely propose subcategories from its assigned records, which we call ``\textit{Discovery}''. To ensure differential privacy, these proposals are filtered using the Gaussian Thresholding Mechanism \citep{desfontaines2021differentiallyprivatepartitionselection} with threshold $\tau_{\text{open}}(\sigma,\delta)=1+\sigma\Phi^{-1}(1-\delta)$, where $\sigma$ is the noise level and $\Phi$ is the standard Gaussian cumulative distribution function. Data-driven discovery can nonetheless miss concepts whose counts fall below the privacy threshold, particularly when the budget is divided across multiple attributes and hierarchical levels. Conditioned on the surviving categories $\mathcal{C}_{\text{open}}$, we therefore prompt the taxonomy model $\mathcal{M}_{\text{tax}}$ to propose mutually exclusive complementary categories $\mathcal{C}_{\text{comp}}$ using its world knowledge about the attribute, dataset description, and parent category $c_p$, which we call ``\textit{Complementary Proposal}''. Records whose initial categories did not survive are reclassified over the combined candidates. The surviving categories $\mathcal{C}_{\text{surv}}$ are attached to $c_p$, and records are allocated for refinement at the next level, while \texttt{other} and \texttt{none} remain terminal. The resulting root-to-leaf paths define the hierarchical domain $\Omega_A$ and the corresponding record path assignments.

\subsection{SHIP} %Structured Hierarchical Insights with Privacy}
\label{subsec:ship}

Based on the taxonomy discovery mechanism introduced above, we introduce \system in \Cref{alg:ship}, consisting of three main components:

\noindent \textbf{Workload Analysis.} For a workload $\mathcal{W}$, \system first analyzes it through a highly capable model $\mathcal{M}_{\text{attr}}$ and finds atomic features needed to answer queries in it. This step returns an attribute set $\mathcal{A}$ for further structured insight extraction.

\begin{algorithm}[!t]
\caption{\system}
\label{alg:ship}
\KwIn{Records $\mathcal{D}=\{x_1,\dots,x_N\}$, Depth $d$, Workload $\mathcal{W}$, Dataset description $S_D$, Privacy budget $(\rho_{\text{aim}},\rho_{\text{tax}},\delta_{\text{tax}})$, Attribute model $\mathcal{M}_{\text{attr}}$, Taxonomy model $\mathcal{M}_{\text{tax}}$, Classification model $\mathcal{M}_{\text{cls}}$.}
\KwOut{Taxonomy $\mathcal{T}$, Insight $\widetilde{\mathcal{D}}$.}

Obtain attribute set $\mathcal{A}\leftarrow\mathcal{M}_{\text{attr}}(\mathcal{W})$\;
Obtain $\rho_{\text{step}}$ and $\delta_{\text{step}}$ via \Cref{eq:step budget}\;

\For{$A\in\mathcal{A}$}{
    $\mathcal{T}_A,\{y_i^{(A)}\}_{i=1}^N \leftarrow$ \Cref{alg:hybrid discovery}\;
}

$\mathcal{T}\leftarrow\bigcup_{A\in\mathcal{A}}\mathcal{T}_A$\;
$\widetilde{\mathcal{D}}\leftarrow\textsc{AIM}(\{(y_i^{(A)})_{A\in\mathcal{A}}\}_{i=1}^N,\mathcal{T},\rho_{\text{aim}})$\;

\Return{$\mathcal{T},\widetilde{\mathcal{D}}$}\;
\end{algorithm}

\noindent \textbf{Taxonomy Building and Category Allocation.} 
Given a set of target attributes $\mathcal{A}$ and a maximum depth $d$, \system builds the hierarchical taxonomy $\mathcal{T}$ by repeatedly applying \Cref{alg:hybrid discovery} for all attributes $A \in \mathcal{A}$. The privacy budget is uniformly allocated across all attributes and hierarchy levels:
\begin{equation} \label{eq:step budget}
\rho_{\text{step}} = \frac{\rho_{\text{tax}}}{|\mathcal{A}|\cdot d}, \qquad \delta_{\text{step}} = \frac{\delta_{\text{tax}}}{|\mathcal{A}|\cdot d}.
\end{equation}
For each attribute $A\in\mathcal{A}$, \Cref{alg:hybrid discovery} returns the attribute taxonomy $\mathcal{T}_A=\{A:\Omega_A\}$ and root-to-leaf record assignments $\{y_i^{(A)}\}_{i=1}^N$ for each attribute $A$. We merge them and obtain a multi-attribute taxonomy $\mathcal{T}$ and a multi-dimensional allocation table $\tilde{D}$.
% \ryan{this would benefit from a figure or example} \kai{todo: add a figure}

\noindent \textbf{Private Tabular Synthesis and Release.}
After taxonomy construction, we can release the DP taxonomy and separately privatized per-attribute histograms. However, the record-level category assignments $\{(y_i^{(A)})_{A \in \mathcal{A}}\}_{i=1}^N$ that capture cross-attribute correlations remain private. We then use the taxonomy $\mathcal{T}$ to define discrete attribute domains and apply a DP synthetic data mechanism, AIM \citep{mckenna2024aimadaptiveiterativemechanism}, with the privacy budget $\rho_{\text{aim}}$. AIM privately measures low-order marginals over the category assignments and estimates a dataset $\widetilde{\mathcal{D}}$ that preserves those marginals well.
The resulting structured insights, consisting of the taxonomy $\mathcal{T}$ and the synthetic dataset $\widetilde{\mathcal{D}}$, can be released to answer queries and conduct downstream analyses. Since these operations rely solely on the released DP outputs, the post-processing property of DP ensures that they do not incur any additional privacy cost.

\subsection{Privacy Guarantee}
\label{subsec:privacy}
In this section, we formally evaluate the privacy guarantee of \system through \Cref{thm:ship_privacy}. The detailed proof is provided in \Cref{sec:appendix proof}.
\begin{theorem}[Privacy Guarantee of \system]
\label{thm:ship_privacy}
Let $\mathcal{D}$ be a private dataset under unbounded DP. We define $\rho = \rho_{\text{tax}} + \rho_{\text{aim}}$ and $\bar{\delta} = \frac{\delta-\delta_{\text{tax}}}{1-\delta_{\text{tax}}}$. Then, for any $\delta_{\text{tax}} < \delta < 1$, \system (\Cref{alg:ship}) satisfies $(\varepsilon,\delta)$-DP, where
\begin{equation*}
\varepsilon = \inf_{\alpha>1}\left\{\alpha\rho+\log\left(1-\frac{1}{\alpha}\right)-\frac{\log\bar{\delta}+\log\alpha}{\alpha-1}\right\}.
\end{equation*}
\end{theorem}

% The privacy guarantee holds by combining the sequential composition of zCDP across the $|\mathcal{A}| \cdot d$ discovery steps and the AIM synthesis phase with Gaussian DP partition selection, followed by conversion to standard $(\varepsilon, \delta)$-DP. The detailed proof and sensitivity analysis are provided in \Cref{sec:appendix proof}.

\begin{table*}[!t]
\centering
\caption{Deterministic query-answering evaluation across single- and multi-attribute questions on {WildChat} and {SynChat} under varying privacy budgets $\varepsilon$. We report average score (1-5 scale $\uparrow$) and accuracy ($\text{Score} \ge 4$, $\% \uparrow$).}
\label{tab:qa_eval}
\small
\setlength{\tabcolsep}{2.2pt} % Scaled to comfortably fit all 14 columns across \textwidth
\begin{tabular}{ll cccccc cccccc}
\toprule
\multirow{3}{*}[-0.9ex]{$\varepsilon\quad\;\;$} & \multirow{3}{*}[-0.9ex]{\textbf{Algorithm}} & \multicolumn{6}{c}{\textbf{WildChat}} & \multicolumn{6}{c}{\textbf{SynChat}} \\
\cmidrule(lr){3-8} \cmidrule(lr){9-14}
 & & \multicolumn{2}{c}{Single-Attr} & \multicolumn{2}{c}{Multi-Attr} & \multicolumn{2}{c}{{Overall}} 
   & \multicolumn{2}{c}{Single-Attr} & \multicolumn{2}{c}{Multi-Attr} & \multicolumn{2}{c}{{Overall}} \\
\cmidrule(lr){3-4} \cmidrule(lr){5-6} \cmidrule(lr){7-8} \cmidrule(lr){9-10} \cmidrule(lr){11-12} \cmidrule(lr){13-14}
 & & Avg Score & Acc & Avg Score & Acc & Avg Score & Acc & Avg Score & Acc & Avg Score & Acc & Avg Score & Acc \\
\midrule
\multirow{4}{*}[-0.3ex]{1.0}
 & \texttt{Urania}                 & 1.27 & 6.0 & 1.24 & 6.0 & 1.26 & 6.0 & 1.27 & 6.7 & 1.10 & 0.0 & 1.18 & 3.3 \\
 & \texttt{Flat-Close}                            & 2.43 & 32.7 & 2.59 & 37.3 & 2.51 & 35.0 & 1.81 & 20.0 & 1.70 & 17.3 & 1.76 & 18.7 \\
 & \texttt{Flat-Open}                               & 3.01 & 46.0 & 3.13 & 44.7 & 3.07 & 45.3 & 2.27 & 30.7 & 1.88 & 21.3 & 2.07 & 26.0 \\
 \cmidrule{2-14}
 & \cellcolor{gray!20}\texttt{\system} (Ours) 
   & \cellcolor{gray!20}\textbf{4.05} & \cellcolor{gray!20}\textbf{72.7} & \cellcolor{gray!20}\textbf{3.56} & \cellcolor{gray!20}\textbf{53.3} & \cellcolor{gray!20}\textbf{3.80} & \cellcolor{gray!20}\textbf{63.0} 
   & \cellcolor{gray!20}\textbf{3.54} & \cellcolor{gray!20}\textbf{59.3} & \cellcolor{gray!20}\textbf{4.41} & \cellcolor{gray!20}\textbf{83.3} & \cellcolor{gray!20}\textbf{3.97} & \cellcolor{gray!20}\textbf{71.3} \\
\midrule
\multirow{4}{*}[-0.3ex]{5.0}
 & \texttt{Urania}                 & 1.43 & 10.7 & 1.16 & 4.0 & 1.30 & 7.3 & 2.29 & 28.0 & 1.53 & 12.7 & 1.91 & 20.3 \\
 & \texttt{Flat-Close}                            & 2.26 & 29.3 & 3.03 & 48.0 & 2.65 & 38.7 & 1.60 & 14.7 & 1.57 & 13.3 & 1.59 & 14.0 \\
 & \texttt{Flat-Open}                               & 2.92 & 43.3 & \textbf{3.33} & \textbf{49.3} & 3.13 & 46.3 & 2.13 & 28.0 & 2.53 & 36.0 & 2.33 & 32.0 \\
 \cmidrule{2-14}
 & \cellcolor{gray!20}\texttt{\system} (Ours) 
   & \cellcolor{gray!20}\textbf{3.60} & \cellcolor{gray!20}\textbf{60.0} & \cellcolor{gray!20}3.31 & \cellcolor{gray!20}46.7 & \cellcolor{gray!20}\textbf{3.46} & \cellcolor{gray!20}\textbf{53.3} 
   & \cellcolor{gray!20}\textbf{3.80} & \cellcolor{gray!20}\textbf{64.7} & \cellcolor{gray!20}\textbf{4.27} & \cellcolor{gray!20}\textbf{79.3} & \cellcolor{gray!20}\textbf{4.04} & \cellcolor{gray!20}\textbf{72.0} \\
\midrule
\multirow{4}{*}[-0.3ex]{10.0}
 & \texttt{Urania}                 & 1.59 & 14.0 & 1.03 & 0.0 & 1.31 & 7.0 & 2.47 & 32.7 & 1.70 & 14.0 & 2.08 & 23.3 \\
 & \texttt{Flat-Close}                            & 2.33 & 30.0 & 3.01 & 47.3 & 2.67 & 38.7 & 1.72 & 17.3 & 1.63 & 15.3 & 1.68 & 16.3 \\
 & \texttt{Flat-Open}                               & 2.89 & 42.7 & 3.18 & 48.0 & 3.03 & 45.3 & 2.17 & 29.3 & 2.52 & 35.3 & 2.34 & 32.3 \\
 \cmidrule{2-14}
 & \cellcolor{gray!20}\texttt{\system} (Ours) 
   & \cellcolor{gray!20}\textbf{3.85} & \cellcolor{gray!20}\textbf{66.0} & \cellcolor{gray!20}\textbf{3.46} & \cellcolor{gray!20}\textbf{51.3} & \cellcolor{gray!20}\textbf{3.66} & \cellcolor{gray!20}\textbf{58.7} 
   & \cellcolor{gray!20}\textbf{3.65} & \cellcolor{gray!20}\textbf{62.0} & \cellcolor{gray!20}\textbf{3.98} & \cellcolor{gray!20}\textbf{72.0} & \cellcolor{gray!20}\textbf{3.82} & \cellcolor{gray!20}\textbf{67.0} \\
\bottomrule
\end{tabular}
\end{table*}

\section{Experiments}

In this section, we briefly introduce our experimental setup, including evaluation metrics, baselines, and datasets, following by our experimental results and findings.  Additional implementation and experimental details are provided in \Cref{app:detail exp set}.

\noindent \textbf{Evaluation Metrics.}
We evaluate the dataset-level statistical and analytical utility of insights through query answering, using two complementary protocols.
\textit{Deterministic QA Accuracy} \citep{zhang2025chatlogscollectiveinsights} evaluates targeted questions with well-defined ground-truth answers, including numerical, comparison, and yes/no questions. We further distinguish between single-feature and multi-feature questions based on the information required to answer them (see \Cref{sec:appendix result} for examples.) An offline agent built on Gemini-3.5-Flash-Lite answers each question using only the generated insight, and the response is evaluated against the ground truth on a discrete 1-5 scale using LLM-as-a-judge.
\textit{Open-Ended QA Quality} evaluates broader analytical questions that require a comprehensive answer (see \Cref{tab:open_report_ship_single}). Following factuality-based evaluation \citep{min2023factscore}, we decompose the offline agent's reports into atomic statements and compute precision (the fraction of insight-derived statements supported by the ground truth), recall (the fraction of ground-truth statements supported by the insight), and $F_1$.

These two protocols provide a common basis for evaluating insights across different output formats, measuring both their ability to recover specific facts and to support higher-level analysis. For structured insights, we additionally introduce \textit{Semantic Sinkhorn Divergence} to assess the fidelity of the extracted tabular representation (details provided in \Cref{sec:app sink div}). We report 1-way divergence over individual attributes and 2-way divergence over attribute pairs, capturing both marginal information and cross-attribute relationships.

\noindent \textbf{Baselines.}
We compare \system with representative well-established differentially private insight generation approaches described in \Cref{sec:existing_paradigms} and enumerated in \Cref{tab:paradigm_taxonomy} (see also \Cref{app:detail exp set}.) The main QA evaluation compares the workload-driven utility of representative private insight pipelines across output formats. The structured fidelity evaluation complements this comparison by examining how taxonomy construction affects information preservation within tabular representations. We therefore include the hierarchical closed- and open-set variants of \system, denoted Hie-Close and Hie-Open, in the structured fidelity evaluation, where they serve as controlled ablations of taxonomy construction. For all methods, we use Gemini-3.5-Flash as the taxonomy and attribute models and Gemini-3.5-Flash-Lite as the classification model.

\noindent \textbf{Datasets.}
For query-answering evaluation, we use GenAI conversation datasets, including WildChat \citep{wildchat} and SynChat, which is a synthetic human-AI conversation dataset constructed on predefined metadata and thus has deterministic ground truth. Structured insight quality is evaluated on WildChat, CFPB \citep{cfpb}, and ChatDoctor \citep{chatdoctor}, covering conversational, financial, and medical domains. We additionally construct SynChat\_topic to study challenging queries involving variable-length structures such as within-session topic transitions.

\begin{figure*}[!t]
\centering
\includegraphics[width=\textwidth]{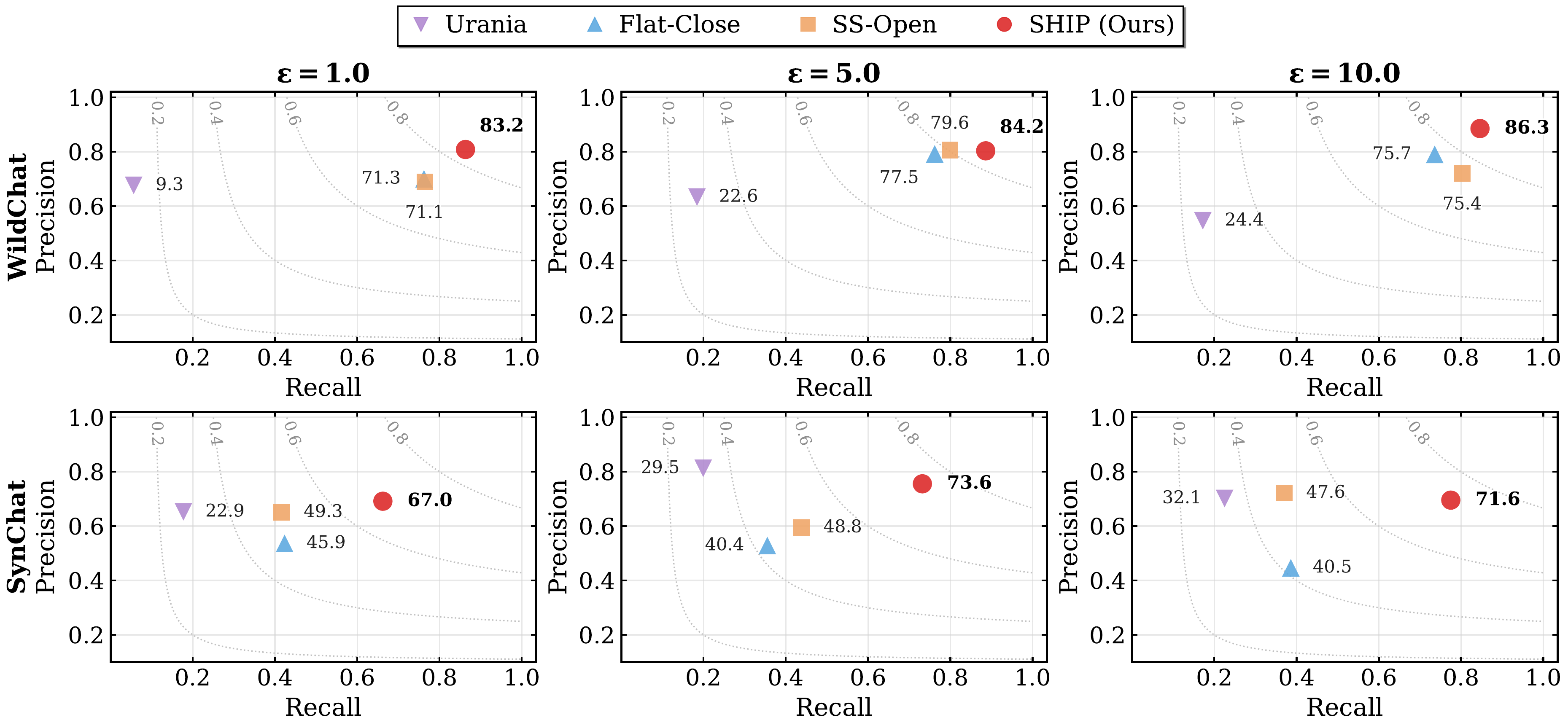}
\caption{Precision, Recall, and $F_1$ in open-ended query-answering evaluation across WildChat (top row) and SynChat (bottom row) under privacy budgets $\varepsilon \in \{1, 5, 10\}$. Dotted background curves denote iso-$F_1$ contours.}
\label{fig:open_qa_eval}
\end{figure*}

\subsection{Main Experiments}
\label{sec:main exp}

% \ryan{the results make it seem like prior work is "bad", but it just wasn't designed to solve the problem we consider in this work. Framing should reflect that}

\noindent \textbf{Deterministic Query-Answering Evaluation.}
\Cref{tab:qa_eval} reports the aggregate results of deterministic query answering averaged across the workload for WildChat and SynChat. To supplement these aggregate results, we provide examples illustrating different queries and the answers obtained by different insights mechanisms in \Cref{sec:appendix result}. \system achieves the strongest overall performance across most tested settings, which is expected because it is specifically designed for this query answering task.  Because \system delivers fine-grained semantic information and cross-attribute relationships, it's able to effectively answer a wider variety of queries than alternative approaches.

The baseline results illustrate how different representations support the evaluated workload. Urania’s cluster-level summaries and keyword statistics emphasize thematic patterns for exploratory analysis, offering less direct support for our analytical queries about fine-grained attributes and cross-attribute joint distributions. This helps explain its lower accuracy, particularly on multi-attribute queries. Flat-Close and Flat-Open expose explicit attributes for statistical analysis, but their flat taxonomies provide less coverage of the semantic distinctions required by detailed queries in this workload.

% \ryan{If we have space, I think it would be worthwhile bringing this table in from the appendix to the main body.}
% \ryan{I would move the pointer to this table to the top of this paragraph, before talking about why Urania and other baselines might have scored poorly on this eval.}

\begin{table*}[!t]
\centering
\caption{Distribution fidelity (Sinkhorn divergence $\downarrow$) across datasets under varying privacy budgets $\varepsilon$.}
\small
\begin{tabular}{ll cc cc cc}
\toprule
\multirow{2}{*}[-0.4ex]{$\varepsilon$} & \multirow{2}{*}[-0.4ex]{\textbf{Algorithm}} & \multicolumn{2}{c}{\textbf{CFPB}} & \multicolumn{2}{c}{\textbf{ChatDoctor}} & \multicolumn{2}{c}{\textbf{WildChat}} \\
\cmidrule(lr){3-4} \cmidrule(lr){5-6} \cmidrule(lr){7-8}
 & & 1-way Div $\downarrow$ & 2-way Div $\downarrow$ & 1-way Div $\downarrow$ & 2-way Div $\downarrow$ & 1-way Div $\downarrow$ & 2-way Div $\downarrow$ \\
\midrule
\multirow{5}{*}{1.0} 
 & \texttt{Flat-Close}        & 0.2215          & 0.3469          & 0.1142          & 0.1917          & 0.1696          & 0.2966          \\
 & \texttt{Flat-Open}           & 0.1581          & 0.2529          & 0.1178          & 0.1831          & 0.1581          & 0.2585          \\
 & \texttt{Hie-Close}         & 0.1465          & 0.2182          & \textbf{0.0745} & 0.1139          & 0.1564          & 0.2353          \\
 & \texttt{Hie-Open}          & 0.1314          & 0.1980          & 0.0900          & 0.1396          & 0.1324          & 0.2155          \\
 \cmidrule{2-8}
 & \cellcolor{gray!20}\texttt{\system} (Ours) & \cellcolor{gray!20}\textbf{0.1313} & \cellcolor{gray!20}\textbf{0.1896} & \cellcolor{gray!20}0.0762 & \cellcolor{gray!20}\textbf{0.1129} & \cellcolor{gray!20}\textbf{0.1201} & \cellcolor{gray!20}\textbf{0.1942} \\
\midrule
\multirow{5}{*}{5.0} 
 & \texttt{Flat-Close}        & 0.2219          & 0.3473          & 0.1142          & 0.1915          & 0.1697          & 0.2966          \\
 & \texttt{Flat-Open}           & 0.1563          & 0.2483          & 0.1167          & 0.1815          & 0.1570          & 0.2564          \\
 & \texttt{Hie-Close}         & 0.1472          & 0.2198          & 0.0746          & 0.1141          & 0.1565          & 0.2355          \\
 & \texttt{Hie-Open}          & \textbf{0.1132} & 0.1675          & 0.0833          & 0.1286          & 0.1246          & 0.2023          \\
 \cmidrule{2-8}
 & \cellcolor{gray!20}\texttt{\system} (Ours) & \cellcolor{gray!20}0.1143 & \cellcolor{gray!20}\textbf{0.1652} & \cellcolor{gray!20}\textbf{0.0744} & \cellcolor{gray!20}\textbf{0.1104} & \cellcolor{gray!20}\textbf{0.1192} & \cellcolor{gray!20}\textbf{0.1915} \\
\midrule
\multirow{5}{*}{10.0} 
 & \texttt{Flat-Close}        & 0.2219          & 0.3472          & 0.1143          & 0.1916          & 0.1697          & 0.2967          \\
 & \texttt{Flat-Open}           & 0.1553          & 0.2460          & 0.1165          & 0.1809          & 0.1574          & 0.2569          \\
 & \texttt{Hie-Close}         & 0.1474          & 0.2200          & \textbf{0.0746} & 0.1141          & 0.1564          & 0.2354          \\
 & \texttt{Hie-Open}          & \textbf{0.1037} & \textbf{0.1534} & 0.0810          & 0.1247          & 0.1217          & 0.1973          \\
 \cmidrule{2-8}
 & \cellcolor{gray!20}\texttt{\system} (Ours) & \cellcolor{gray!20}0.1114 & \cellcolor{gray!20}0.1630 & \cellcolor{gray!20}0.0755 & \cellcolor{gray!20}\textbf{0.1133} & \cellcolor{gray!20}\textbf{0.1175} & \cellcolor{gray!20}\textbf{0.1876} \\
\bottomrule
\end{tabular}%
\label{tab:div}
\end{table*}

\noindent \textbf{Open-ended Query-Answering Evaluation.}
\Cref{fig:open_qa_eval} presents the results of open-ended analytical query answering across WildChat and SynChat. The reports generated from several baseline resources achieve relatively high precision but substantially lower recall, indicating that they capture only a subset of the reference findings. For example, on SynChat at $\varepsilon = 1.0$, Urania achieves 65.3\% precision and 17.7\% recall; similarly, at $\varepsilon = 10.0$, Flat-Open achieves 72.2\% precision and 37.0\% recall. The low recalls indicates that some important conclusions for queries that we can obtain from ground truth are not captured by these insights.

\system achieves broader coverage: on WildChat, its recall consistently ranges from 84.0\% to 89.0\%, while on SynChat it also achieves substantially higher recall than the competing methods. This pattern is consistent with the complementary information made available by its integrated design, which combines hierarchical hybrid taxonomy construction with private tabular synthesis to expose fine-grained semantic categories and cross-attribute relationships. Together with its precision scores, this broader coverage yields the highest $F_1$ scores across all evaluation settings, indicating the strongest combined performance on factual correctness and finding coverage.

\noindent \textbf{Structured Insight Fidelity Evaluation.} Finally, we compare the synthetic structured insight with structured insights generated by other structured insight paradigms, shown in \Cref{tab:div}. First, the empirical results demonstrate that hierarchical approaches are more effective than flat structures at preserving the original data distribution. Comparing the flat algorithms (Flat-Close and Flat-Open) with the hierarchical algorithms (Hie-Close, Hie-Open, and our proposed \system), we observe that the flat baselines consistently yield higher divergences across all datasets and privacy budgets. 

Second, among hierarchical approaches, no single open- or closed-set strategy dominates across all datasets and privacy budgets. Hie-Close relies on model priors for taxonomy construction, yielding stable categories but limiting data-adaptive discovery. For instance, on CFPB, its 1-way divergence remains around 0.14 as the privacy budget increases, but \system's decreases from 0.13 to 0.11. In contrast, Hie-Open adapts its categories directly to the private data, but becomes more sensitive to small privacy budgets because low-frequency categories are harder to recover through private discovery. This effect is more pronounced on diverse datasets such as WildChat and ChatDoctor. By combining data-driven discovery with prior-based supplementation, \system provides a more robust compromise.

\subsection{Challenging Queries for \system}

We further examine queries that are difficult for \system to answer correctly. On SynChat\_topic, where conversations may traverse multiple topics, we ask \system to capture each session's topic-switching path at $\varepsilon=10$. As shown in \Cref{fig:topic number}, the released distribution is biased toward shorter paths and fails to recover longer sequences. This discrepancy arises before tabular synthesis: variable-length paths yield a large, sparse categorical domain, making long sequences difficult to recover via private domain discovery, while prior-based completion is also ill-suited to such highly data-dependent combinations.

\begin{figure}[!t]
\centering
\includegraphics[width=\columnwidth]{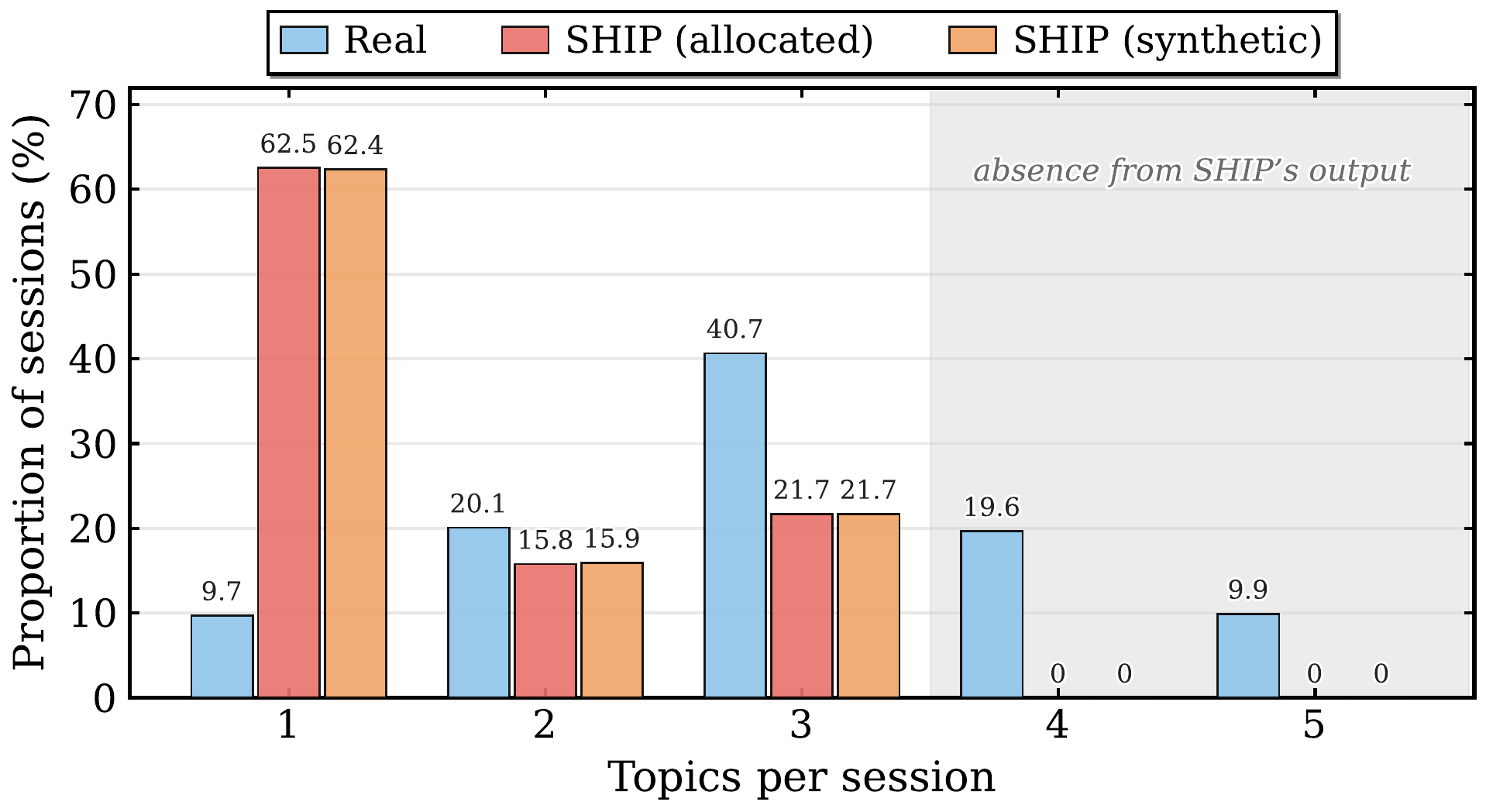}
\caption{Distribution of the number of topics per session in SynChat\_topic, obtained from real dataset and \system ($\varepsilon=10$). ``Allocated'' is \system's per-record category assignment, ``synthetic'' refers to the output of its DP tabular synthesizer.}
\label{fig:topic number}
\end{figure}

This example highlights the importance of matching the insight representation to the structure required by the analytical workload. Encoding complete sequences as atomic categories can make private discovery difficult, even when their constituent topics are common. More broadly, workloads involving ordering, relationships between entities, or open-ended compositional patterns motivate representations that preserve these structures without requiring every complete configuration to be discovered as a separate category. Exploring such representations is a promising direction for extending private insight generation to a broader range of analytical workloads.

\section{Conclusion}

In this work, we study DP insight extraction from textual corpora through a query-answering perspective. We introduce \system, a framework for extracting structured insights under DP, and develop a unified evaluation protocol that measures how well different insight representations support common analytical needs. Our experiments show that structured insights can provide strong query-answering utility and also identify queries that remain challenging for current insight-generation paradigms, highlighting opportunities for richer representations and extraction mechanisms. We hope this work provides a common basis for developing and evaluating future private insight systems.

\section*{Limitations}

Our work has two main limitations. First, our query evaluation focuses primarily on dataset-level statistical and analytical queries, including counting, comparison, aggregation, and multi-attribute analysis, and does not establish effectiveness for arbitrary queries. Even within this scope, our experiments identify queries that SHIP’s current DP structured representation does not support well, such as those involving long topic-switching sequences. These findings highlight the need for alternative representations or extraction mechanisms when the required information is difficult to capture through categorical attributes.

Second, we evaluate insights primarily through their ability to answer a specified analytical workload. This captures one dimension of utility and does not establish \system's superiority across all aspects of data analysis; other representations may offer complementary benefits beyond this evaluation. Moreover, both our ground-truth annotations and \system also use three-level taxonomies, which may favor \system in reference-based comparisons. We have not quantified this potential bias. Future work should assess robustness to alternative taxonomy depths and reference schemas, alongside broader analytical benefits such as discovering unanticipated patterns and supporting evolving information needs.

\section*{Ethical Considerations}

Our experiments use publicly available datasets and synthetically generated data, without collecting additional private user data. SHIP is intended to support aggregate analysis rather than identify individual users. We do not release identifiable user records or information, and the insights released by SHIP are generated under DP.

% \section*{Acknowledgments}

% This document has been adapted
% by Steven Bethard, Ryan Cotterell and Rui Yan
% from the instructions for earlier ACL and NAACL proceedings, including those for
% ACL 2019 by Douwe Kiela and Ivan Vuli\'{c},
% NAACL 2019 by Stephanie Lukin and Alla Roskovskaya,
% ACL 2018 by Shay Cohen, Kevin Gimpel, and Wei Lu,
% NAACL 2018 by Margaret Mitchell and Stephanie Lukin,
% Bib\TeX{} suggestions for (NA)ACL 2017/2018 from Jason Eisner,
% ACL 2017 by Dan Gildea and Min-Yen Kan,
% NAACL 2017 by Margaret Mitchell,
% ACL 2012 by Maggie Li and Michael White,
% ACL 2010 by Jing-Shin Chang and Philipp Koehn,
% ACL 2008 by Johanna D. Moore, Simone Teufel, James Allan, and Sadaoki Furui,
% ACL 2005 by Hwee Tou Ng and Kemal Oflazer,
% ACL 2002 by Eugene Charniak and Dekang Lin,
% and earlier ACL and EACL formats written by several people, including
% John Chen, Henry S. Thompson and Donald Walker.
% Additional elements were taken from the formatting instructions of the \emph{International Joint Conference on Artificial Intelligence} and the \emph{Conference on Computer Vision and Pattern Recognition}.


\bibliography{ref}

\appendix
\crefalias{section}{appendix}
\crefalias{subsection}{appendix}

\begin{table*}[!t]
\centering
\small
\caption{The summary of four investigated datasets. Mid-tokens refer to the median number of tokens, and 95\%-tokens mean the 95\% percentile of token numbers. The token size statistics are calculated using \texttt{Sentence-T5-base}.}
\label{tab:dataset_summary_detailed}
\begin{tabularx}{0.96\textwidth}{llllX}
\toprule
\textbf{Dataset} & \textbf{Records} & \textbf{Mid-Tokens} & \textbf{95\%-Tokens} & \textbf{Open-ended Workload} \\
\midrule
WildChat   & 838k & 1,061 & 4,463 & \{What are the users' primary task intents and topic domains, what specific constraints or output formatting requirements are requested, and what types of solutions or response strategies do the assistants provide to resolve the requests?\} \\
\midrule
ChatDoctor & 112k &   252 &   531 & \{What are the users' primary symptoms and chief complaints, and for those symptoms, what diagnostic tests or examinations do the doctors suggest, and what are the doctors' diagnosis conclusions?\} \\
\midrule
CFPB       &  10k &   192 &   885 & \{What are the primary problems or grievances encountered by the users, and what specific inquiries, requests, or corrective actions do the users seek to resolve the issues?\} \\
\midrule
SynChat    &  10k &   217 &   636 & [Same as WildChat] \\
\midrule
SynChat\_topic & 10k & 532 & 879 & \{Which topics come up in each conversation session? For sessions covering several topics, what are the most common complete paths from the first topic to the last, and how often does each occur?\} \\
\bottomrule
\end{tabularx}
\end{table*}

\begin{algorithm}[!t]
\caption{Complete Hierarchical Hybrid DP Taxonomy Building}
\label{alg:hybrid discovery full}
\KwIn{Records $\mathcal{D}=\{x_1,\dots,x_N\}$, Dataset description $S_D$, Max depth $d$, Attribute $A$, Privacy parameters $(\rho_{\text{step}},\delta_{\text{step}})$, Taxonomy model $\mathcal{M}_{\text{tax}}$, Classification model $\mathcal{M}_{\text{cls}}$.}
\KwOut{Taxonomy $\mathcal{T}_A$, Internal record paths $\{y_i\}_{i=1}^N$.}

$\rho_{\text{open}}=0.9\rho_{\text{step}},\quad \rho_{\text{close}}=0.1\rho_{\text{step}}$\;
$\sigma_{\text{open}}=1/\sqrt{2\rho_{\text{open}}},\quad \sigma_{\text{close}}=1/\sqrt{2\rho_{\text{close}}}$\;
Compute $\tau_{\text{open}}$ and $\tau_{\text{close}}$\;
Initialize tree-like domain $\Omega_A$ with root $\bot$\;

\For{$l \leftarrow 1$ \KwTo $d$, each level $l-1$ category $c_p$ in $\Omega_A$}{
    Skip if $c_p$ is \texttt{other} or \texttt{none}\;
    $\mathcal{D}_p \leftarrow$ records assigned to $c_p$\;

    \ForEach{$x_i\in\mathcal{D}_p$}{
        $y_i \leftarrow \mathcal{M}_{\text{cls}}(x_i,A,c_p)$\;
    }

    $\mathcal{C}_{\text{open}} \leftarrow \emptyset$\;
    \ForEach{$c\in\{y_i\}_{x_i\in\mathcal{D}_p}$}{
        $\tilde n_c \leftarrow \sum_i\mathbf{1}\{y_i=c\}+\mathcal N(0,\sigma_{\text{open}}^2)$\;
        \If{$\tilde n_c\ge\tau_{\text{open}}(\sigma_{\text{open}},\delta_{\text{step}})$}{
            $\mathcal{C}_{\text{open}} \leftarrow \mathcal{C}_{\text{open}}\cup\{c\}$\;
        }
    }

    $\mathcal{C}_{\text{comp}} \leftarrow \mathcal{M}_{\text{tax}}(A,S_D,c_p,\mathcal{C}_{\text{open}})$\;
    \ForEach{$x_i\in\mathcal{D}_p$}{
        \If{$y_i\notin\mathcal{C}_{\text{open}}$}{
            $y_i \leftarrow \mathcal{M}_{\text{cls}}(x_i,\mathcal{C}_{\text{open}}\cup\mathcal{C}_{\text{comp}})$\;
        }
    }

    $\mathcal{C}_{\text{surv}} \leftarrow \{\texttt{other}\}$\;
    \ForEach{$c\in\mathcal{C}_{\text{open}}\cup\mathcal{C}_{\text{comp}}$}{
        $\tilde n_c \leftarrow \sum_i\mathbf{1}\{y_i=c\}+\mathcal N(0,\sigma_{\text{close}}^2)$\;
        \If{$\tilde n_c\ge\tau_{\text{close}}$}{
            $\mathcal{C}_{\text{surv}} \leftarrow \mathcal{C}_{\text{surv}}\cup\{c\}$\;
        }
    }
    Map excluded assignments to \texttt{other}\;
    Add $\mathcal{C}_{\text{surv}}$ as children of $c_p$ in $\Omega_A$ and allocate $\mathcal{D}_p$ using $\{y_i\}$\;
}
De-duplicate and refine categories in $\Omega_A$\;
\Return{$\{A:\Omega_A\},\{y_i\}_{i=1}^N$}\;
\end{algorithm}

\section{Supplementary Related Work}

In this section, we discuss more about the relationship between \system and other related works, mainly including DP text analysis and DP synthetic data generation.

\noindent\textbf{Related Text Analytics Paradigms.}
RAG-based methods answer queries by retrieving relevant passages from the underlying corpus, with recent work extending this paradigm to private retrieval \citep{koga2025privacypreservingretrievalaugmentedgenerationdifferential,wu2025privateragansweringmultiplequeries,mori-etal-2026-differentially} and large-scale conversational analysis \citep{zhang2025chatlogscollectiveinsights}. Structured analytics systems instead use Text-to-SQL, semantic operators, or taxonomy induction to execute analytical workloads over unstructured data \citep{yu2019spiderlargescalehumanlabeleddataset,Lam_2024,patel2025semanticoperatorsdeclarativemodel,anderson2024designllmpoweredunstructuredanalytics,shankar2025docetlagenticqueryrewriting,sadia2026squidsynthesizingrelationaldatabases}. Unlike these query-time analysis paradigms, we focus on releasing a reusable insight representation under DP, which can support various downstream anaysis.

\noindent\textbf{DP Synthetic Text Generation.}
Recent DP synthetic text frameworks achieve state-of-the-art utility by extracting structured insights as an intermediate representation rather than replacing direct insight extraction. Specifically, ACTG-ARL \citep{hu2025actg} and MAPLE \citep{chien2026maple} leverage the structured insights to steer DP fine-tuning or private evolution to generate synthetic text. DP synthetic text generation targets a broader release that can support downstream tasks beyond a predefined analytical workload, whereas we focus on compact representations for data insight. The two directions are therefore complementary.

\section{Detailed Experimental Settings}
\label{app:detail exp set}

\subsection{Detailed Implementation of \system}

\noindent \textbf{Detailed Taxonomy Building.} We provide a complete version of \Cref{alg:hybrid discovery} here in \Cref{alg:hybrid discovery full}. It introduces how we conduct the private partition selection to discover categories and allocate supplementary categories after that. There are several main differences between it and \Cref{alg:hybrid discovery}. 

First, we will skip those parent categories that are ``none'' or ``other'', which refer to the case that the attribute is not mentioned or the real category is too rare, respectively. After allocation, all filtered categories will be renamed to ``other'' as a default category. Second, we have a filter threshold for supplementary categories in order to reduce domain complexity. In our experiment, we set $\tau_{\text{close}}=10$ for simplicity. Finally, we also utilize an LLM to de-duplicate categories across levels before we output the taxonomy.

\noindent \textbf{Detailed Privacy Budget Split.} Given a target $(\varepsilon,\delta)$-DP guarantee, we set the privacy parameters $(\rho_{\text{aim}},\rho_{\text{tax}},\delta_{\text{tax}})$ in \Cref{alg:ship} as follows:
\[
\begin{aligned}
    \delta_\text{tax} &= \delta/3,\\
    \rho_\text{aim} &= \rho_\text{tax} = \frac{1}{2}f_{\rho}(\varepsilon, 2\delta/3),
\end{aligned}
\]
Here, $f_{\rho}$ numerically inverts the zCDP-to-DP conversion to obtain a zCDP budget compatible with the target DP parameters.

\subsection{Detailed Baseline Introduction}

In our work, we mainly consider 5 insight algorithms that can provide strict DP guarantees. We acknowledge that there are other algorithms, but most of them cannot provide DP protection during their pipelines so we omit them in our experiments. Here is a detailed introduction of them, including introduction and implementation details. 

\begin{enumerate}[leftmargin=*, label=\textbullet]
    \item \textbf{Urania.} Urania is a cluster-based DP insight extraction framework \citep{liu2025uraniadifferentiallyprivateinsights}, which can also be regarded as a DP version of Clio \citep{tamkin2024clioprivacypreservinginsightsrealworld}. First, an LLM extracts semantic facets (intent summaries and keywords) for each record. Next, dense sentence embeddings of the summaries are computed and clustered via $k$-means to partition the dataset. For each cluster, a differentially private keyword histogram is constructed via the Gaussian/Laplacian mechanism to privately select top keywords and noisy cluster counts. Finally, an LLM synthesizes cluster titles and high-level descriptions conditioned strictly on the DP keyword histograms and subsampled record summaries. In our work, we further optimize this pipeline by involving workload into the prompt and incorporating hierarchical clustering.

    \item \textbf{Flat-Close.} This approach represents a set of methods, including non-DP structured insight extraction used by OpenAI \citep{chatterji2025chatgpt} and DP structured metadata synthesis used in Hu et al.'s work \citep{hu2025actg}. In this method, we first define a taxonomy by an LLM without any access to the raw data, and annotate each data. Finally, a tabular synthesis mechanism is employed to output a privatized structured insight.
    
    \item \textbf{Flat-Open.} This baseline serves as a direct extension of existing open-set structured summarization methods \citep{cheu2025provablyprivateanalyticsinsights}. Existing literature predominantly formulates open-set discovery along isolated attribute dimensions, releasing categories that survive differential private partition selection alongside independent 1-dimensional noisy histograms. To preserve baseline fidelity without introducing ad-hoc synthetic modeling unstudied in prior work, our implementation evaluates each attribute independently and constructs multi-attribute insights via separate marginal releases. This faithful setting directly captures the structural limitations of current open-set approaches when applied to joint, multi-attribute analytical workloads.
    
    \item \textbf{Hie-Close \& Hie-Open.} These two baselines are variants of \system. Hie-Close relies completely on a pre-defined taxonomy by removing the open discovery step in \system, and Hie-Open discovers categories only by partition selection by removing the closed supplementary category generation step in \system.
\end{enumerate}

\subsection{Datasets Information}

There are five datasets investigated in our experiments, including three real-world datasets and two artificial datasets. For workload, we generate most content-related queries and form the workload descriptions. A summarization is provided in \Cref{tab:dataset_summary_detailed}. 
% Because WildChat and ChatDoctor contain more than 100k samples, considering the price of calling API and experiment efficiency, we apply Poisson sampling and privacy amplification with sample rates $r = 2\%$ and $10\%$, respectively. 
Following established text-insight systems that rely on subsampling to reconcile large-scale corpora with LLM API token limits (e.g., Clio \citep{tamkin2024clioprivacypreservinginsightsrealworld} and Urania \citep{liu2025uraniadifferentiallyprivateinsights}), we apply uniform Poisson subsampling on corpora exceeding 100k records (sampling ratios $r = 2\%$ for WildChat and $r = 10\%$ for ChatDoctor). Beyond computational feasibility, Poisson subsampling is theoretically grounded in differential privacy via privacy amplification by subsampling \citep{balle2018privacyamplificationsubsamplingtight}: executing an $(\epsilon_{\text{sub}}, \delta_{\text{sub}})$-DP mechanism on a Poisson subsample with sampling ratio $r$ satisfies global $(\epsilon, \delta)$-DP on the full corpus, where $\epsilon = \ln(1 + r(e^{\epsilon_{\text{sub}}} - 1))$ and $\delta = r \cdot \delta_{\text{sub}}$. The reported privacy parameter $\epsilon$ denotes the unified global end-to-end privacy budget. 

The two artificial datasets are generated label-first: we sample the ground truth and then synthesize text conditioned on it, so the true annotation of every record is known by construction rather than estimated after the fact. The pipeline has three stages. First, for each annotation attribute, an LLM builds a hierarchical taxonomy from the attribute description and a description of the target corpus, which we optionally prune. Second, we sample each record's ground-truth labels from that taxonomy under a controlled distribution; this is the only stage at which the two corpora differ. SynChat draws one hierarchical category path per attribute, autoregressively conditioned on an anchor attribute, so that correlations between attributes are planted rather than incidental. SynChat\_topic instead draws a sequence of leaves per record, a walk whose length is drawn from an explicit distribution and whose steps follow a transition matrix we control, so that a conversation changes topic a known number of times and in a known direction. Finally, an LLM renders the sampled labels into a conversation.

\subsection{QA Metric Introduction}

\noindent \textbf{Workload Access}. We assume that all methods have the same access to the workload files during insight generation, which are shown in \Cref{tab:dataset_summary_detailed}.   

\noindent \textbf{Auxiliary Agent.} To perform closed-loop evaluation, we deploy an auxiliary agent implemented on the Antigravity SDK \citep{antigravity2026} with the Gemini-3.5-Flash-Lite model. The agent is explicitly constrained to operate in a controlled offline setting: default capabilities and external retrieval tools are completely masked, leaving only an isolated Python execution environment with bounded file access to prevent data leakage and ensure fair benchmarking.

\begin{table*}[!t]
\centering
\small
\caption{Several examples of deterministic queries on SynChat with standard answers. We present one example for each query category (single- or multi-attributes $\times$ query type).}
\label{tab:query_answer_examples_alt}
\renewcommand{\arraystretch}{1.4}
\begin{tabularx}{\textwidth}{p{3.2cm} X}
\toprule
\textbf{Category} & \textbf{Query Content} \\
\midrule
Single-Attribute Yes/No & \textbf{Query:} Across all user interactions, do direct information retrieval and text transformation tasks collectively constitute more than 65\% of all primary user intents? \newline
\textbf{Answer:} Yes. \\
\midrule
Single-Attribute Comparison & \textbf{Query:} Across user conversation intents, do users more frequently ask the assistant to modify, summarize, or transform existing text, or to generate original creative writing pieces? \newline
\textbf{Answer:} Transforming and modifying existing text. \\
\midrule
Single-Attribute Ask for Number & \textbf{Query:} Within user requests aimed at modifying and transforming existing text, what proportion specifically request content summarization into condensed formats such as bullet point extractions, executive summaries, or single-sentence abstracts? \newline
\textbf{Answer:} 0.3994 (39.94\%). \\
\midrule
Multi-Attribute Yes/No & \textbf{Query:} Among user requests focused on software coding tasks (such as code generation, bug fixing, performance optimization, and unit testing), does the computer programming topic domain account for more than 35\% of all conversations? \newline
\textbf{Answer:} Yes. \\
\midrule
Multi-Attribute Comparison & \textbf{Query:} When users engage with topics in computer programming, do their requests more frequently involve text transformation tasks (such as summarization, translation, paraphrasing, and proofreading) or creative writing tasks (such as narrative fiction, poetry, scriptwriting, and roleplay)? \newline
\textbf{Answer:} Text transformation tasks. \\
\midrule
Multi-Attribute Ask for Number & \textbf{Query:} In discussions focused on the mathematics domain (covering geometry, calculus, number theory, and probability and statistics), what proportion of assistant responses adopt a clarification-seeking strategy (such as disambiguation, slot filling, preference elicitation, or intent confirmation)? \newline
\textbf{Answer:} 0.2003 (20.03\%). \\
\bottomrule
\end{tabularx}
\end{table*}

\noindent \textbf{Deterministic QA.} Deterministic evaluation queries are generated by our auxiliary agent from structured summaries of the raw text. Following prior work \citep{zhang2025chatlogscollectiveinsights}, we first apply a non-DP summarization algorithm to distill the raw corpus. Conditioned on an open-ended user intent, the agent generates deterministic queries paired with canonical ground-truth answers. Specifically, we generate 100 queries across each of three distinct question types: \textit{Yes/No}, \textit{Comparative}, \textit{Numerical}, with each type evenly partitioned between single-feature and multi-feature queries. This yields 300 deterministic queries per dataset. Several examples are shown in \Cref{tab:query_answer_examples_alt}.

For WildChat, since we do not have ground-truth metadata, we generate it using an LLM-summarization algorithm. Each record is independently categorized using a powerful LLM, and all categories are aggregated and de-duplicated. We also conduct this process in a hierarchical structure. We acknowledge that this process may introduce bias due to pipeline design, model capacity, and randomness, which is also a limitation of prior work \citep{zhang2025chatlogscollectiveinsights}.

The queries are answered using an offline agent. The agent is provided with the full set of generated insight artifacts and tasked with answering each benchmark query through closed-loop data inspection. To evaluate correctness, an LLM judge compares the agent's predicted response against the predefined ground-truth answer, assigning a discrete score on a 1-to-5 scale. Because these queries have standard and deterministic answers, we only evaluate once for each algorithm.

\begin{table*}[!t]
\centering
\small
\caption{Case study from SynChat ($\varepsilon=10$) illustrating categorization across algorithms. \system discovers hierarchical paths across all four orthogonal dimensions, whereas \texttt{Flat-Close} and \texttt{Flat-Open} use 1-level labels, and \texttt{Urania} outputs an unstructured text cluster descriptor.}
\label{tab:synchat_multirow_case_study}
\renewcommand{\arraystretch}{1.25}
\begin{tabularx}{\textwidth}{p{0.48\textwidth} X}
\toprule
\textbf{Raw Dialogue} & \textbf{Algorithm \& Categorization} \\
\midrule

\multirow{22}{0.48\textwidth}{%
\textbf{user}: Can you proofread this calculus passage and fix the punctuation errors, while also organizing your final explanations using a markdown table? Here is the draft: The derivative in differential calculus measures the rate of change. It derives from the Latin word differentia meaning difference. \par\medskip

\textbf{ai}: Here is the corrected text with proper punctuation, followed by a breakdown of the etymological origin presented as a markdown table: \par\smallskip
The derivative, in differential calculus, measures the rate of change. It derives from the Latin word \textit{*differentia*}, meaning difference. \par\smallskip
| Term | Language of Origin | Etymological Meaning | \par
| :--- | :--- | :--- | \par
| Derivative | Latin (*derivativus*) | To draw off, lead down, or derive from a source | \par
| Differential | Latin (*differentia*) | A variance, distinction, or difference between quantities | \par\medskip

\textbf{user}: That table is very clear, thank you. Could you double-check if any extra punctuation marks are missing around the word differentia? \par\medskip

\textbf{ai}: The punctuation is correct as written. To be thorough, the term differential calculus itself originates from the Latin \textit{*differentia*}, where the root relates to the infinitesimal differences studied in calculus.
}

& \textbf{Ground Truth} \par
\textbullet~Intent: \texttt{text transformation > proofreading > punctuation correction} \par
\textbullet~Topic: \texttt{mathematics > calculus > differential calculus} \par
\textbullet~Constraint: \texttt{output\_format > tabular\_format > markdown\_table} \par
\textbullet~Strategy: \texttt{direct\_information > factual\_definition > etymological\_origin} \\
\cmidrule{2-2}
& \textbf{\system} \par
\textbullet~Intent: \texttt{text editing > proofreading \& correction > punctuation correction} \par
\textbullet~Topic: \texttt{stem > mathematics in stem > mathematical calculus} \par
\textbullet~Constraint: \texttt{formatting and structure > tabular formatting > markdown table} \par
\textbullet~Strategy: \texttt{direct answer > correction > linguistic correction} \\
\cmidrule{2-2}
& \textbf{Flat-Close} \par
\textbullet~Intent: \texttt{text processing} \par
\textbullet~Topic: \texttt{stem} \par
\textbullet~Constraint: \texttt{formatting and structure} \par
\textbullet~Strategy: \texttt{structured presentation} \\
\cmidrule{2-2}
& \textbf{Flat-Open} \par
\textbullet~Intent: \texttt{text editing} \par
\textbullet~Topic: \texttt{stem} \par
\textbullet~Constraint: \texttt{formatting} \par
\textbullet~Strategy: \texttt{step-by-step explanation} \\
\cmidrule{2-2}
& \textbf{Urania} \par
\textbullet~Cluster: \texttt{Technical Mathematics, Precise Translation, and Backend Development} \par
\textbullet~Keywords: \texttt{mathematical definition, literal translation, ...} \par
\textbullet~Description: \texttt{User queries spanning advanced mathematics, precise linguistic translation, and backend development. Users seek rigorous mathematical definitions in differential geometry and calculus, literal translations under length constraints, and ...} \\
\bottomrule
\end{tabularx}
\end{table*}

\begin{table*}[!t]
\centering
\small
\caption{Case studies of deterministic query evaluation on SynChat ($\varepsilon=10$) across different structural paradigms. We provide representative examples: (1) \textit{Macro Shallow Query}: Top-level distributions are accurately captured by all methods. (2) \textit{Hierarchical Query}: Fine-grained subcategories are correctly resolved by \system via hierarchical taxonomy, whereas flat baselines collapse and fail. (3) \textit{Cross-Attribute Joint Query}: Tabular representations with orthogonal attribute axes (\system and \texttt{Flat-Close}) accurately evaluate bivariate cross-tabulations, whereas 1D semantic clustering (\texttt{Urania}) entangles dimensions and fails. (4) \textit{Highly Detailed Query}: Excessively fine-grained joint conditions lead to failure across all methods.}
\label{tab:query_case_studies}
\renewcommand{\arraystretch}{1.25}
\setlength{\tabcolsep}{4pt}
\begin{tabularx}{\textwidth}{p{6.2cm} X X X X}
\toprule
\textbf{Query \& Ground Truth} & \textbf{\system{} (Ours)} & \textbf{Flat-Close} & \textbf{Flat-Open} & \textbf{Urania} \\

\midrule
\textbf{Query:} Across the entire dataset, when users specify constraints on their requests, do they more frequently demand specific output formatting requirements or enforce explicit length limits? \vspace{0.3em}\newline
\textbf{Answer:} Output formatting &
\textbf{Formatting rules} \newline {\textbf{Score: 5/5}} \checkmark &
\textbf{Formatting rules} \newline {\textbf{Score: 5/5}} \checkmark &
\textbf{Formatting rules} \newline  {\textbf{Score: 5/5}} \checkmark &
\textbf{Formatting rules} \newline  {\textbf{Score: 5/5}} \checkmark \\

\midrule
\textbf{Query:} Within web development conversations in computer programming, what proportion of user requests specify tabular presentation formats in their output constraints? \vspace{0.3em}\newline
\textbf{Answer:} 16.07\%  &
\textbf{15.91\%} \newline  {\textbf{Score: 5/5}} \checkmark &
\textit{``I don't know''} \newline {\textbf{Score: 1/5}} $\times$ &
\textit{``I don't know''} \newline {\textbf{Score: 1/5}} $\times$ &
\textit{``I don't know''} \newline {\textbf{Score: 1/5}} $\times$ \\

\midrule
\textbf{Query:} Across user requests with a primary intent of software coding, what proportion specify output format constraints? \vspace{0.3em}\newline
\textbf{Answer:} 41.77\% &
\textbf{43.15\%} \newline {\textbf{Score: 5/5}} \checkmark &
\textbf{39.80\%} \newline {\textbf{Score: 5/5}} \checkmark &
\textbf{21.40\%} \newline {\textbf{Score: 2/5}} $\times$ &
\textbf{62.80\%} \newline {\textbf{Score: 2/5}} $\times$  \\ 

\midrule 
\textbf{Query:} Within web development topic discussions under computer programming, what proportion of queries concentrate on frontend component development and UI construction? \vspace{0.3em}\newline
\textbf{Answer:} 6.73\% &
\textit{``I don't know''} \newline {\textbf{Score: 1/5}} $\times$ &
\textit{``I don't know''} \newline {\textbf{Score: 1/5}} $\times$ &
\textit{``I don't know''} \newline {\textbf{Score: 1/5}} $\times$ &
\textit{``I don't know''} \newline {\textbf{Score: 1/5}} $\times$ \\
\bottomrule
\end{tabularx}
\end{table*}

\begin{table*}[!t]
\centering
\caption{Open-ended query-answering evaluation across analytical conclusion statements on {WildChat} and {SynChat} under varying privacy budgets $\varepsilon$. We report Recall ($\text{Score} \ge 4$, $\% \uparrow$), Precision ($\text{Score} \ge 4$, $\% \uparrow$), and primary $F_1$ score ($\text{Score} \ge 4$, $\% \uparrow$), together with their standard deviations.}
\label{tab:open_qa_eval}
\small
\setlength{\tabcolsep}{6.0pt}
\begin{tabular}{ll ccc ccc}
\toprule
\multirow{2}{*}[-0.6ex]{$\varepsilon\quad\;\;$} & \multirow{2}{*}[-0.6ex]{\textbf{Algorithm}} & \multicolumn{3}{c}{\textbf{WildChat}} & \multicolumn{3}{c}{\textbf{SynChat}} \\
\cmidrule(lr){3-5} \cmidrule(lr){6-8}
 & & Recall & Precision & $F_1$ & Recall & Precision & $F_1$ \\
\midrule
\multirow{4}{*}[-0.3ex]{1.0} 
 & \texttt{Urania}                                     & $5.6_{\pm 4.3}$ & $67.6_{\pm 20.7}$ & $9.3_{\pm 6.4}$ & $17.7_{\pm 21.7}$ & $65.3_{\pm 19.1}$ & $22.9_{\pm 22.4}$ \\
 & \texttt{Flat-Close}                              & $76.2_{\pm 10.4}$ & $70.0_{\pm 17.3}$ & $71.3_{\pm 10.0}$ & $42.3_{\pm 14.6}$ & $53.7_{\pm 5.3}$ & $45.9_{\pm 8.7}$ \\
 & \texttt{Flat-Open}                                 & $76.4_{\pm 9.0}$ & $68.9_{\pm 17.7}$ & $71.1_{\pm 12.7}$ & $41.6_{\pm 13.8}$ & $65.1_{\pm 16.8}$ & $49.3_{\pm 13.1}$ \\
 & \system\ (Ours)                                   & $\textbf{86.3}_{\pm 5.1}$ & $\textbf{80.8}_{\pm 8.5}$ & $\textbf{83.2}_{\pm 4.5}$ & $\textbf{66.2}_{\pm 12.9}$ & $\textbf{69.2}_{\pm 14.3}$ & $\textbf{67.0}_{\pm 11.7}$ \\
\midrule
\multirow{4}{*}[-0.3ex]{5.0} 
 & \texttt{Urania}                                     & $18.4_{\pm 21.0}$ & $63.3_{\pm 17.2}$ & $22.6_{\pm 23.3}$ & $19.9_{\pm 12.8}$ & $\textbf{81.3}_{\pm 6.6}$ & $29.5_{\pm 16.9}$ \\
 & \texttt{Flat-Close}                              & $76.2_{\pm 9.5}$ & $79.2_{\pm 7.8}$ & $77.5_{\pm 7.8}$ & $35.5_{\pm 20.0}$ & $52.9_{\pm 8.0}$ & $40.4_{\pm 16.4}$ \\
 & \texttt{Flat-Open}                                 & $79.9_{\pm 6.0}$ & $\textbf{80.6}_{\pm 16.1}$ & $79.6_{\pm 9.8}$ & $43.8_{\pm 7.5}$ & $59.5_{\pm 15.9}$ & $48.8_{\pm 5.9}$ \\
 & \system\ (Ours)                                   & $\textbf{88.6}_{\pm 7.3}$ & $80.3_{\pm 5.3}$ & $\textbf{84.2}_{\pm 5.9}$ & $\textbf{73.2}_{\pm 8.3}$ & $75.6_{\pm 7.2}$ & $\textbf{73.6}_{\pm 2.1}$ \\
\midrule
\multirow{4}{*}[-0.3ex]{10.0} 
 & \texttt{Urania}                                     & $17.2_{\pm 10.9}$ & $54.6_{\pm 18.3}$ & $24.4_{\pm 13.5}$ & $22.5_{\pm 14.4}$ & $70.2_{\pm 11.6}$ & $32.1_{\pm 18.7}$ \\
 & \texttt{Flat-Close}                              & $73.6_{\pm 13.3}$ & $79.0_{\pm 7.6}$ & $75.7_{\pm 9.1}$ & $38.6_{\pm 11.9}$ & $44.7_{\pm 6.8}$ & $40.5_{\pm 8.7}$ \\
 & \texttt{Flat-Open}                                 & $80.3_{\pm 6.0}$ & $72.0_{\pm 9.4}$ & $75.4_{\pm 5.9}$ & $37.0_{\pm 14.7}$ & $\textbf{72.2}_{\pm 9.0}$ & $47.6_{\pm 14.1}$ \\
 & \system\ (Ours)                                   & $\textbf{84.6}_{\pm 1.5}$ & $\textbf{88.5}_{\pm 9.5}$ & $\textbf{86.3}_{\pm 5.3}$ & $\textbf{77.5}_{\pm 12.5}$ & $69.6_{\pm 10.2}$ & $\textbf{71.6}_{\pm 3.1}$ \\
\bottomrule
\end{tabular}
\end{table*}

\begin{table*}[!t]
\centering
\small
\caption{Compacted analytical reports generated by \system and \texttt{Flat-Close} ($\varepsilon=10.0$) on SynChat and representative extracted atomic statements (Round 1 of evaluation). We showcase representative statements demonstrating macro distributions, fine-grained leaf details, and cross-attribute interactions.}
\label{tab:open_report_ship_single}
\renewcommand{\arraystretch}{1.2}
\begin{tabularx}{\textwidth}{p{0.48\textwidth} p{0.48\textwidth}}
\toprule
\textbf{Generated Report (Compacted)} & \textbf{Extracted Statement Examples} \\
\midrule
\textbf{\system:} \par\smallskip
\texttt{\#\# Executive Summary} \par
\texttt{Empirical investigation reveals a structured ecosystem of user inquiries, cognitive intents, domain distributions, and assistant strategies ...} \par\smallskip
\texttt{\#\# Macro Domain Breakdown} \par
\texttt{* **User Intent:** Anchored in **Information Seeking** ($\approx 43\%$), followed by **Text Editing** ($\approx 22\%$), **Creative Writing** ($\approx 18\%$), and **Coding** ($\approx 15\%$) ...} \par\smallskip
\texttt{\#\# In-Depth Single- \& Multi-Attribute Analysis} \par
\texttt{* **Intent-Domain Synergy:** STEM-driven information seeking represents the single largest cross-attribute intersection in the dataset ...} \par\smallskip
\texttt{\#\# Strategic Takeaways} \par
\texttt{1. **Technical Specialization:** STEM ($\approx 68\%$) and Information Seeking ($\approx 43\%$) mandate rigorous factual verification and step-by-step derivations ...}
&
\noindent \textbullet\; Information seeking accounts for roughly 43\% of all conversational interactions. \par\smallskip
\noindent \textbullet\; STEM disciplines command roughly 68\% of all inquiries. \par\smallskip
\noindent \textbullet\; Within computer science topics, software engineering and database administration dominate the volumetric load. \par\smallskip
\noindent \textbullet\; In mathematics, geometry and calculus inquiries drive detailed analytical requests. \par\smallskip
\noindent \textbullet\; STEM-driven information seeking represents the single largest intersection in the dataset. \par\smallskip
\noindent \textbullet\; Technical coding tasks frequently incorporate rigorous formatting constraints, such as requiring code blocks or diff patches, alongside computational constraints. \par\smallskip
... \textit{(69 total decomposed statements)}
\\
\midrule
\textbf{Flat-Close:} \par\smallskip
\texttt{\#\# Executive Summary} \par
\texttt{Investigation into chat conversations reveals dialogues spanning cognitive goals, thematic domains, and operational constraints ...} \par\smallskip
\texttt{\#\# Macro Domain Breakdown} \par
\texttt{* **User Intent:** Information seeking and content generation represent the dominant volume drivers, accounting for nearly half of all prompts ...} \par\smallskip
\texttt{\#\# In-Depth Single- \& Multi-Attribute Analysis} \par
\texttt{* **Cross-Tabulation:** Coding prompts within STEM almost universally trigger executable code delivery with step-by-step debugging explanations ...} \par\smallskip
\texttt{\#\# Strategic Takeaways} \par
\texttt{1. **Programmable Interfaces:** High prevalence of formatting constraints indicates users view chat assistants as programmable tools ...}
&
\noindent \textbullet\; Information seeking and content generation represent the most dominant volume drivers, accounting for nearly half of all user prompts. \par\smallskip
\noindent \textbullet\; Technical and professional domains (STEM and productivity) dominate the thematic distribution across topical domains. \par\smallskip
\noindent \textbullet\; Formatting and structure constraints, such as markdown tables and JSON schemas, represent a major constraint category. \par\smallskip
\noindent \textbullet\; When users initiate programming tasks within the STEM domain, assistants almost universally respond with executable code blocks and step-by-step explanations. \par\smallskip
\noindent \textbullet\; Creative or professional writing within business or humanities simultaneously invokes tone/style and length constraints in roughly two-thirds of cases. \par\smallskip
... \textit{(59 total decomposed statements)}
\\
\bottomrule
\end{tabularx}
\end{table*}

\begin{table*}[!t]
\centering
\small
\caption{Compacted analytical reports generated by \texttt{Flat-Open} and \texttt{Urania} ($\varepsilon=10.0$) on SynChat and representative extracted atomic statements (Round 1 of evaluation). We showcase representative statements illustrating open extracted distributions, entangled macro clusters, and keyword-level granular topics.}
\label{tab:open_report_ss_urania}
\renewcommand{\arraystretch}{1.2}
\begin{tabularx}{\textwidth}{p{0.48\textwidth} p{0.48\textwidth}}
\toprule
\textbf{Generated Report (Compacted)} & \textbf{Extracted Statement Examples} \\
\midrule
\textbf{Flat-Open:} \par\smallskip
\texttt{\#\# Executive Summary} \par
\texttt{Dialogues are anchored in information seeking and creative writing, with STEM serving as the thematic backbone capturing over two-thirds of allocations ...} \par\smallskip
\texttt{\#\# Single-Attribute Distribution Breakdown} \par
\texttt{* **User Intent:** Information Seeking accounts for $\approx 30.3\%$, followed by Creative Writing ($\approx 17.1\%$), Reasoning ($\approx 11.0\%$), and Code Generation ($\approx 10.6\%$) ...} \par\smallskip
\texttt{\#\# Multi-Attribute Relationships \& Cross-Tabulations} \par
\texttt{* **Intent-Domain Dynamics:** Within Code Generation and Debugging, the STEM domain captures over $85\%$ of instances ...} \par\smallskip
\texttt{\#\# Strategic Takeaways} \par
\texttt{1. **Ecosystem Specialization:** Rigorous technical problem-solving (STEM + Code Generation) coexists with extensive creative workflows ...}
&
\noindent \textbullet\; Information seeking represents the single largest intent category, accounting for approximately 30.3\% of all interactions. \par\smallskip
\noindent \textbullet\; STEM topics dominate the entire conversation corpus, accounting for approximately 68.8\% of all topic assignments. \par\smallskip
\noindent \textbullet\; Formatting constraints govern approximately 43.1\% of all interactions where user constraints are specified. \par\smallskip
\noindent \textbullet\; Within code generation and debugging, the STEM domain captures over 85\% of instances. \par\smallskip
\noindent \textbullet\; Humanities topics capture roughly 20\% to 22\% of creative and editing workflows. \par\smallskip
... \textit{(67 total decomposed statements)}
\\
\midrule
\textbf{Urania:} \par\smallskip
\texttt{\#\# Executive Summary} \par
\texttt{The dataset demonstrates an orientation toward complex problem-solving where technical computing, formal sciences, and creative tasks intersect ...} \par\smallskip
\texttt{\#\# Macro Domain Breakdown} \par
\texttt{* **Root Cluster 94 ($\approx 29.4\%$):** Software Engineering, Mathematics, Linguistics, and Creative Writing (merges poetry, Python scripts, and transliteration) ...} \par\smallskip
\texttt{\#\# Granular Topic \& Keyword Analysis} \par
\texttt{* **Mathematical Reasoning:** Prominence of `euclidean geometry` ($\approx 243.5$) and `differential calculus` ($\approx 176.5$), coupled with `markdown table` ($\approx 151.5$) ...} \par\smallskip
\texttt{\#\# Strategic Takeaways} \par
\texttt{1. **Multi-Constraint Integration:** Users rarely separate technical computation from presentation formatting; math queries demand code and markdown tables ...}
&
\noindent \textbullet\; Software engineering, mathematics, linguistics, and creative writing represent the largest single macro-domain at approximately 29.4\% of total volume. \par\smallskip
\noindent \textbullet\; Advanced mathematics, technical computing, and structured formatting account for approximately 26.5\% of all interactions. \par\smallskip
\noindent \textbullet\; Mathematical inquiries in Root Cluster 1 focus on Euclidean and analytic geometry, differential calculus, and the Riemann zeta function. \par\smallskip
\noindent \textbullet\; Enterprise data architecture topics in Root Cluster 63 include PostgreSQL database migration, schema design, and query optimization. \par\smallskip
\noindent \textbullet\; Key thematic drivers in technical workflow clusters include version control, git reset, backend API development, and web state management. \par\smallskip
... \textit{(47 total decomposed statements)}
\\
\bottomrule
\end{tabularx}
\end{table*}

\noindent \textbf{Open-Ended QA.} The pipeline of open-ended QA evaluation is a symmetric, two-pass verification process designed to assess both the completeness and factual fidelity of dataset summarizations. In the first pass (Recall), an autonomous code agent synthesizes a comprehensive analytical report from the ground-truth summarization conditioned on a workload query, which is subsequently decomposed into independent atomic factual statements. These ground-truth statements are then validated against the target summarization by executing closed-loop verification code on the target data files, calculating the proportion of ground-truth findings accurately captured (Recall). Symmetrically, in the second pass (Precision), the agent generates an analytical report directly from the target summarization, decomposes it into candidate statements, and validates them against the ground truth to quantify the factual accuracy of the target's conclusions (Precision). Finally, the F1 score is computed to provide a unified benchmark score balancing insight coverage and factual grounding. We set the LLM temperature to 0 to reduce generation variability. To account for residual variability observed in long-form analytical reports, we run the open-ended QA evaluation five times on the released insights and report the average scores.

\subsection{Textual Table Sinkhorn Divergence.} 
\label{sec:app sink div}

To evaluate how faithfully the synthetic table $\mathcal{D}_{\mathrm{syn}}$ preserves the categorical distributions and attribute correlations of the ground-truth table $\mathcal{D}_{\mathrm{real}}$, we measure the semantic distance between empirical distributions via entropic optimal transport. Note that here, to make our ground-truth $\mathcal{D}_{\mathrm{real}}$ more accurate, we use phrases or sentences summarized by an LLM as its categories.

For a given attribute (or attribute pair in 2-way divergence evaluation), let $\mathbf{p}$ and $\mathbf{q}$ denote the empirical category frequency distributions over unique vocabularies. We map categories into dense semantic representations and construct a ground cost matrix $\mathbf{C} \in [0, 1]^{m \times n}$ via normalized cosine distance, 
\[
    C_{ij} = 1 - \frac{1}{2}(\cos(e_i, e_j) + 1).
\]
For two-way divergence, we define 
\[
    C_{ij} = 1 - \frac{1}{4}(\cos(e^{(1)}_i, e^{(1)}_j) + 1)(\cos(e^{(2)}_i, e^{(2)}_j) + 1).
\]
We then compute the entropic optimal transport cost:
\begin{equation*}
\mathcal{W}_\varepsilon(\mathbf{p}, \mathbf{q}) = \min_{\mathbf{T} \in \Pi(\mathbf{p}, \mathbf{q})} \langle \mathbf{T}, \mathbf{C} \rangle - \varepsilon H(\mathbf{T}),
\end{equation*}
where $\Pi(\mathbf{p}, \mathbf{q})$ is the transport polytope and $H(\mathbf{T})$ is the entropy regularizer ($\varepsilon = 0.1$). To eliminate entropic self-bias, we report the debiased Sinkhorn divergence:
\begin{equation*}
\mathcal{S}_\varepsilon(\mathbf{p}, \mathbf{q}) = \mathcal{W}_\varepsilon(\mathbf{p}, \mathbf{q}) - \frac{1}{2}\mathcal{W}_\varepsilon(\mathbf{p}, \mathbf{p}) - \frac{1}{2}\mathcal{W}_\varepsilon(\mathbf{q}, \mathbf{q}).
\end{equation*}
We evaluate $\mathcal{S}_\varepsilon$ across both 1-column marginal distributions and 2-column joint attribute distributions, reporting the average score across columns (lower indicates superior distributional alignment).

\subsection{Implementation Details}

We implement \system with depth $d=3$ by default. We set $\tau_\text{close}=10$ for simplicity. All algorithms are implemented on the Gemini API. We use Gemini-3.5-Flash-Lite for all classification and keyword extraction tasks, and use Gemini-3.5-Flash as taxonomy generators. The supplementary categories are generated via Simula \citep{davidson2026reasoningdrivensyntheticdatageneration}. For all algorithm, we set global $\varepsilon in \{1, 5, 10\}$ and $\delta = 1/(n \ln(n))$.

\section{Supplementary Experiment Results}
\label{sec:appendix result}

\subsection{Examples of Categorization}

In \Cref{tab:synchat_multirow_case_study}, we present a raw data record with its ground-truth meta data and the categorization results obtained from different algorithms ($\varepsilon=1$). From this table, the first observation is that \system provides a more granular classification than Flat-Open and Flat-Close. For instance, \system propose ``mathematical calculus'' as the topic, while Flat-Open and Flat-Close only categorize the data into ``STEM''. 

In terms of Urania, we can observe that the cluster information has its own strengths and weaknesses. It can identify ``differential geometry and calculus'', which is a key information that all other algorithms do not include. However, the cluster also contain some necessary information, such as translation and backend development. We believe this is a core limitation of cluster-based approaches, which cannot distinguish semantically similar records.

\subsection{Deterministic QA Results}

We provide several representative examples in \Cref{tab:query_case_studies} to help illustrate the capacity of different approaches. The first query is asking about macro information, which can be easily answered by all methods. The second query is asking for detailed information, and thus only \system provides a relatively accuracte answer. The third query is a cross-attribute query. Therefore, only methods that have strict tabular modeling can answer it. The final example is a highly detailed query. This query is asking a tailor distribution information under a sub-category, which is overly detailed than the capacity of any methods. Therefore, all methods fail to return an accurate answer.

\begin{table*}[!t]
\centering
\small
\caption{Case study from SynChat\_topic ($\varepsilon=10$). The conversation moves through several unrelated domains. The second column includes the real categorizations and those by \system; ``$\rightarrow$'' denotes a topic switch.}
\label{tab:synchat_switch_case_study}
\renewcommand{\arraystretch}{1.25}
\begin{tabularx}{\textwidth}{p{0.48\textwidth} p{0.48\textwidth}}
\toprule
\textbf{Raw Dialogue (Compressed)} & \textbf{Algorithm \& Categorization} \\
\midrule

\textbf{user:} I am looking at an H-NMR spectrum for an unknown organic compound with the molecular formula C4H8O2. There is a triplet at 1.2 ppm, a quartet at 4.1 ppm, a singlet at 2.0 ppm \ldots\; Can you help me piece this together? \par\smallskip
\ldots\par\smallskip
\textbf{user:} On a completely different matter, we are heading into trial preparations for a breach of contract case, and opposing counsel just served a massive batch of unorganized emails. What is the best way to handle this electronic discovery dump? \par\smallskip
\ldots\par\smallskip
\textbf{user:} This legal wrangling is tedious \ldots\; Let's talk about something entirely different. What are the key turning points in the early political career of Alexander Hamilton? 
\par\medskip

&
\textbf{Ground Truth} \par
\textbullet~Topic path (level-1 categories only): \texttt{chemistry $\rightarrow$ litigation $\rightarrow$ history} \par
% \textbullet~Topic count: \texttt{3}

\par\medskip\medskip\medskip
\textbf{\system} \par
\textbullet~Topic path (level-1 categories only): \texttt{chemistry $\rightarrow$ law $\rightarrow$ history} \par
% \textbullet~Topic count: \texttt{3} 
% \par\smallskip
\\

\midrule

\textbf{user:} I have a monolithic function in my Python backend \ldots How should I break this down to follow the single responsibility principle? \par
\ldots\par\smallskip
\textbf{user:} While sorting out those backend services, I also need to record our petty cash expenditures and reconcile the February bank statement. \par
\ldots\par\smallskip
\textbf{user:} Speaking of discrepancies, my Diels-Alder reaction keeps yielding the wrong stereoisomer. How can I control endo and exo selectivity? \par
\ldots\par\smallskip
\textbf{user:} That chemistry problem is getting too complicated. Instead, I need to draft a request to the judge to throw out the plaintiff's complaint. \par
\ldots\par\smallskip
\textbf{user:} This legal drafting is overwhelming \ldots drop this entire line of questioning. How do modern historians analyze the biases in 19th-century colonial archives?
\par\medskip

&

\textbf{Ground Truth} \par
\textbullet~Topic path (level-1 categories only): \texttt{programming $\rightarrow$ finance
$\rightarrow$ chemistry $\rightarrow$ litigation $\rightarrow$ history} \par
% \textbullet~Topic count: \texttt{5}
% \par\medskip\noindent\rule{\linewidth}{0.4pt}\par\smallskip

\par\medskip\medskip\medskip

\textbf{\system} \par
\textbullet~Topic path (level-1 categories only): \texttt{programming} \par
% \textbullet~Topic count: \texttt{1}
% \par\smallskip\noindent\rule{\linewidth}{1.0pt}\par\smallskip

\\
\bottomrule
\end{tabularx}
\end{table*}

\subsection{Open-Ended QA Results}

The full numerical results of \Cref{fig:open_qa_eval} is provided in \Cref{tab:open_qa_eval}. To better illustrate how we conduct the open-ended QA evaluation, we present several open-ended query answer report examples and the statements extracted from them in \Cref{tab:open_report_ship_single} and \Cref{tab:open_report_ss_urania}. 

As reflected in the reports and extracted statements, each method captures different structural characteristics of the conversational dataset. \system provides high-level macro distributions across orthogonal axes (e.g., Information Seeking at $\approx 43\%$, STEM disciplines at $\approx 68\%$), resolves fine-grained subcategories within hierarchical leaves (e.g., software engineering and database administration under computer science; geometry and calculus under mathematics), and delineates bivariate cross-attribute intersections (e.g., STEM-driven information seeking, coding workflows governed by structural formatting constraints). 

Flat-Close and Flat-Open capture 1-level macro distributions and bivariate cross-tabulations (e.g., technical STEM queries triggering code delivery), but remain restricted to top-level categories without resolving hierarchical sub-branches. Note that Flat-Open does not model the joint distribution. Therefore, its performance in multivariate QA is highly dependent on whether the attributes are not strongly related.

Finally, Urania represents the workload via 1D semantic clustering, yielding composite macro-clusters that jointly entangle topics, intents, and formatting rules (e.g., Root Cluster 94 combining creative writing, Python scripting, and length constraints at $\approx 29.4\%$) alongside keyword-level frequency statistics for specific technical operations.

\begin{table}[!t]
\centering
\footnotesize
\caption{Average category counts and allocated records across attributes on WildChat.}
\label{tab:wildchat_category_breakdown_compact}
\renewcommand{\arraystretch}{1.15}
\setlength{\tabcolsep}{2.5pt}
\begin{tabular}{cc ccc}
\toprule
\textbf{$\varepsilon$} & \textbf{Level} & \textbf{Open-Set} & \textbf{Close-Set} & \textbf{Other \& None} \\
\midrule
\multirow{3}{*}{$1.0$} 
  & L0 & 13.5 (83.9\%) & 4.5 (2.5\%) & 1.8 (13.7\%) \\
  & L1 & 35.8 (80.2\%) & 43.0 (8.3\%) & 23.0 (11.5\%) \\
  & L2 & 29.8 (46.7\%) & 152.0 (37.6\%) & 74.0 (15.7\%) \\
\midrule
\multirow{3}{*}{$5.0$} 
  & L0 & 16.3 (86.5\%) & 4.3 (1.0\%) & 1.8 (12.5\%) \\
  & L1 & 52.8 (89.3\%) & 37.3 (5.8\%) & 24.8 (5.0\%) \\
  & L2 & 55.0 (55.5\%) & 137.5 (32.1\%) & 93.0 (12.3\%) \\
\midrule
\multirow{3}{*}{$10.0$} 
  & L0 & 20.8 (86.3\%) & 2.8 (0.6\%) & 2.3 (13.1\%) \\
  & L1 & 86.3 (86.1\%) & 30.3 (4.6\%) & 29.5 (9.4\%) \\
  & L2 & 70.8 (51.7\%) & 147.3 (33.0\%) & 119.8 (15.4\%) \\
\bottomrule
\end{tabular}
\end{table}

\begin{table*}[!t]
\centering
\caption{Record-level pairwise preference evaluation of \system\ against each baseline across four datasets under varying privacy budgets $\varepsilon$. Each cell is the average judge preference score on a 1-5 scale oriented toward \system: $5$ strongly prefers \system, $3$ is neutral, and $1$ strongly prefers the baseline, so values above $3$ favour \system\ ($\uparrow$). Score denotes the workload-unaware judgement and $\text{Score}_w$ the workload-aware judgement.}
\label{tab:pairwise_preference_eval}
\small
\setlength{\tabcolsep}{5.0pt}
\begin{tabular}{ll cc cc cc cc}
\toprule
\multirow{2}{*}[-0.6ex]{$\varepsilon\quad\;\;$} & \multirow{2}{*}[-0.6ex]{\textbf{Comparison}} & \multicolumn{2}{c}{\textbf{WildChat}} & \multicolumn{2}{c}{\textbf{SynChat}} & \multicolumn{2}{c}{\textbf{CFPB}} & \multicolumn{2}{c}{\textbf{ChatDoctor}} \\
\cmidrule(lr){3-4} \cmidrule(lr){5-6} \cmidrule(lr){7-8} \cmidrule(lr){9-10}
 & & Score & $\text{Score}_w$ & Score & $\text{Score}_w$ & Score & $\text{Score}_w$ & Score & $\text{Score}_w$ \\
\midrule
\multirow{3}{*}[-0.3ex]{1.0}
 & \system\ vs.\ \texttt{Urania}      & $4.570$ & $4.714$ & $4.786$ & $4.781$ & $4.208$ & $4.236$ & $4.541$ & $4.689$ \\
 & \system\ vs.\ \texttt{Flat-Close}  & $3.912$ & $3.754$ & $4.158$ & $4.155$ & $4.122$ & $4.182$ & $3.435$ & $3.757$ \\
 & \system\ vs.\ \texttt{Flat-Open}   & $3.793$ & $3.768$ & $4.136$ & $4.159$ & $3.893$ & $4.083$ & $3.645$ & $3.851$ \\
\midrule
\multirow{3}{*}[-0.3ex]{5.0}
 & \system\ vs.\ \texttt{Urania}      & $4.478$ & $4.687$ & $4.819$ & $4.772$ & $4.481$ & $4.427$ & $4.463$ & $4.603$ \\
 & \system\ vs.\ \texttt{Flat-Close}  & $4.004$ & $3.944$ & $4.418$ & $4.413$ & $4.315$ & $4.378$ & $3.499$ & $3.839$ \\
 & \system\ vs.\ \texttt{Flat-Open}   & $3.887$ & $3.908$ & $4.267$ & $4.293$ & $4.149$ & $4.290$ & $3.716$ & $3.964$ \\
\midrule
\multirow{3}{*}[-0.3ex]{10.0}
 & \system\ vs.\ \texttt{Urania}      & $4.474$ & $4.639$ & $4.787$ & $4.745$ & $4.548$ & $4.428$ & $4.390$ & $4.564$ \\
 & \system\ vs.\ \texttt{Flat-Close}  & $3.882$ & $3.754$ & $4.281$ & $4.233$ & $4.375$ & $4.437$ & $3.443$ & $3.803$ \\
 & \system\ vs.\ \texttt{Flat-Open}   & $3.898$ & $3.887$ & $4.218$ & $4.230$ & $4.195$ & $4.355$ & $3.740$ & $3.962$ \\
\bottomrule
\end{tabular}
\end{table*}

\subsection{Contributions of Open-Set and Closed-Set Categories}
\label{app:cat proportion}

\Cref{tab:wildchat_category_breakdown_compact} reports the average number of categories and their corresponding record proportions across attributes, with percentages shown in parentheses. At shallower levels (L0 and L1), open-set categories account for over 80\% of records, while closed-set categories account for less than 10\%. At the deepest level (L2), however, the contribution of closed-set categories increases substantially to 32.1--37.6\%, while that of open-set categories drops to 46.7--55.5\%. These results suggest that open-set partition selection effectively captures prevalent categories at coarse granularities, whereas closed-set supplementation becomes increasingly important for maintaining record coverage at finer granularities.

\subsection{Challenging Query Examples}

\Cref{tab:synchat_switch_case_study} contrasts two conversations from SynChat\_topic, both annotated by \system in the same run at $\varepsilon=10$. The pair is matched: both are cross-domain topic switches drawn from the same taxonomy and released under the same schema, so the only variable is the number of topics the user visits. Ground truth is defined over level-1 categories, so we compare \system's release at level 1 as well, discarding the subcategory suffix it appends below that level.

In the first conversation, the user moves from NMR spectroscopy to electronic discovery to Alexander Hamilton's early political career, and \system recovers all three topics in the correct order. The one apparent mismatch, ``{law}'' where the ground truth reads
``{litigation}'', is a difference of granularity rather than of referent:
``{law}'' is exactly the parent domain of ``{litigation}'' in the
generating taxonomy, so both names denote the same segment. The second conversation visits five domains under the same switching mechanics, and \system only releases the single topic ``{programming}'': four of the five leave no trace in the output.

\subsection{Preference Evaluation Results}

We additionally adopt the summary preference metric used in Urania's evaluation \citep{liu2025uraniadifferentiallyprivateinsights}. Instead of scoring each output in isolation, this metric presents an LLM judge with the outputs of two algorithms for the same record and elicits a relative preference. We compare \system\ against each baseline in turn on a $1$-$5$ scale oriented toward \system, where $5$ strongly prefers \system, $3$ is neutral, and $1$ strongly prefers the baseline, so values above $3$ favour \system. We elicit two judgements per record: the workload-aware judgement asks which output is more valuable for answering the workload query, and the workload-unaware judgement asks which output better describes the raw text. 

Structured insight methods release tabular metadata whose correspondence with the raw records is altered by privatization. To compare all methods in a common format, we first use an LLM to verbalize each tabular record as a natural-language summary. We acknowledge that this conversion may affect the fairness of the evaluation. Urania already releases natural-language cluster summaries and is used as is. We randomise the presentation order on every call to control for position bias and evaluate every record of each dataset. Results are reported in \Cref{tab:pairwise_preference_eval}. 

All results lie above the neutral value of $3$, so \system\ is preferred over every baseline, on every dataset, at every privacy budget, and under both criteria. The baseline ordering is stable: {Flat-Close} and {Flat-Open} are the strongest competitors and differ little from one another, while \texttt{Urania} is the weakest everywhere. The gap to {Urania} is largely one of granularity rather than annotation quality, since it assigns each record the summary of its enclosing cluster while the structured methods emit record-specific descriptions. This also explains why the gap varies with the corpus: it is large on the diverse {WildChat} and {ChatDoctor} texts but nearly vanishes on the narrow, homogeneous complaints of {CFPB}, where a shared cluster summary is almost as informative as a record-level one. The two flat baselines are therefore the stricter points of comparison.

% Two further trends are visible. First, the advantage of \system\ over the flat baselines peaks at the intermediate budget $\varepsilon=5$ on {WildChat}, {SynChat} and {ChatDoctor}, suggesting that its hierarchical label space needs a sufficient budget to survive the injected noise but yields less relative benefit once the budget is generous enough for flat releases to become accurate on their own; {CFPB}, with its narrower topical range, instead improves monotonically with $\varepsilon$. Second, the two judgement criteria agree closely on {WildChat} and {SynChat}, which share a broad workload, whereas the workload-aware score is clearly higher on \textsc{ChatDoctor} and {CFPB}, whose workloads are tightly scoped; the benefit of workload-conditioned extraction thus appears to grow with the specificity of the analytical query.

\begin{table}[!t]
\centering
\footnotesize
\caption{Deterministic QA results with different QA agent models on SynChat.}
\label{tab:agent_model_ablation}
\renewcommand{\arraystretch}{0.9}
\setlength{\tabcolsep}{2.5pt}
\begin{tabular*}{\columnwidth}{@{\extracolsep{\fill}}ccccc@{}}
\toprule
\multirow{2}{*}{\textbf{$\varepsilon$}} & \multicolumn{2}{c}{\textbf{Gemini}} & \multicolumn{2}{c}{\textbf{Sonnet}} \\
\cmidrule(lr){2-3}\cmidrule(lr){4-5}
& {Avg Score} & {Acc (\%)} & {Avg Score} & {Acc (\%)} \\
\midrule
1.0  & 3.97 & 71.3 & 4.00 & 71.0 \\
5.0  & 4.04 & 72.0 & 3.99 & 71.7 \\
10.0 & 3.82 & 67.0 & 3.81 & 68.0 \\
\bottomrule
\end{tabular*}
\end{table}

\begin{table*}[!t]
\centering
\small
\caption{Example QA agent execution trace involving scope resolution and category aggregation across taxonomy levels. Resource loading and unrelated code and outputs are omitted.}
\label{tab:agent_example}
\renewcommand{\arraystretch}{1.25}
\setlength{\tabcolsep}{0.01\textwidth}
\begin{tabularx}{\textwidth}{p{0.48\textwidth} p{0.48\textwidth}}
\toprule
\multicolumn{2}{p{0.96\textwidth}}{
\textbf{Query:} Among user conversation constraints specifying output formatting, do list format requests (such as bulleted lists, numbered lists, checklists, and description lists) exceed 45\% of all output formatting constraints?
} \\
\midrule
\textbf{Code} & \textbf{STDOUT} \\
\midrule

\multicolumn{2}{l}{\textbf{Step 1: Identify the reference population and inspect its categories.}} \\[3pt]
\begin{minipage}[t]{\linewidth}
\raggedright\ttfamily\scriptsize
formatting\_rows = df[\par
\hspace*{1em}df['user\_constraints (level 1)'] == 'formatting'\par
]\par
print(f"Total formatting constraint rows: \{len(formatting\_rows)\}")\par
print(\par
\hspace*{1em}formatting\_rows['user\_constraints (level 2)']\par
\hspace*{1em}.value\_counts(dropna=False)\par
)\par
print(\par
\hspace*{1em}formatting\_rows['user\_constraints (level 3)']\par
\hspace*{1em}.value\_counts(dropna=False)\par
)
\end{minipage}
&
\begin{minipage}[t]{\linewidth}
\raggedright\ttfamily\scriptsize
Total formatting constraint rows: 4677\par
\medskip
user\_constraints (level 2)\par
bullet points: 241\par
numbered list formatting: 107\par
checklist formatting: 100\par
[other categories omitted]\par
\medskip
user\_constraints (level 3)\par
bulleted list: 349\par
numbered list structured output: 247\par
general bullet points: 178\par
checklist structured output: 100\par
markdown checklist: 91\par
numbering style constraint: 73\par
[other categories omitted]
\end{minipage}
\smallskip
\\
\midrule

\multicolumn{2}{l}{\textbf{Step 2: Combine matching categories across levels and calculate the proportion.}} \\[3pt]
\begin{minipage}[t]{\linewidth}
\raggedright\ttfamily\scriptsize
df\_formatting = formatting\_rows\par
list\_formatting\_count = len(df\_formatting[\par
\hspace*{1em}df\_formatting['user\_constraints (level 2)'].isin([\par
\hspace*{2em}'bullet points',\par
\hspace*{2em}'numbered list formatting',\par
\hspace*{2em}'checklist formatting'\par
\hspace*{1em}]) |\par
\hspace*{1em}df\_formatting['user\_constraints (level 3)']\par
\hspace*{1em}.str.contains(\par
\hspace*{2em}'bullet|list|checklist|numbered',\par
\hspace*{2em}case=False, na=False\par
\hspace*{1em})\par
])\par
print(\par
\hspace*{1em}f"List formatting constraints count: "\par
\hspace*{1em}f"\{list\_formatting\_count\} / "\par
\hspace*{1em}f"\{len(formatting\_rows)\} = "\par
\hspace*{1em}f"\{list\_formatting\_count/len(formatting\_rows)*100:.2f\}\%"\par
)
\end{minipage}
\smallskip
&
\begin{minipage}[t]{\linewidth}
\raggedright\ttfamily\scriptsize
List formatting constraints count: 1144 / 4677 = 24.46\%
\end{minipage}
\\
\midrule

\multicolumn{2}{p{0.96\textwidth}}{
\textbf{Final answer:} No. List format requests account for 1,144 out of 4,677 output formatting constraint rows in the released DP synthetic table (24.46\%), which does not exceed 45\%.
} \\
\bottomrule
\end{tabularx}
\end{table*}

\subsection{Robustness to the QA Agent Model}

For experimental convenience, our main experiments use models from the Gemini family for both insight extraction and QA evaluation. To examine whether using the same model family in both stages substantially affects our evaluation, we replace the QA agent model with Claude Sonnet and repeat the deterministic QA evaluation on SynChat. We use the same released insight resources and retain all other evaluation settings, so that this comparison isolates the effect of the QA agent model. As shown in \Cref{tab:agent_model_ablation}, replacing Gemini with Claude Sonnet changes the overall deterministic QA score by at most 0.05 and accuracy by at most 1.0\% across the three privacy budgets. These results suggest that \system’s deterministic QA performance on SynChat is not substantially affected by this change in the QA agent model.

\begin{figure*}[!t]
\centering
\begin{tcolorbox}[
    colback=gray!10,          % 10% gray background
    colframe=gray!40,         % Subtle border
    arc=2mm,                  % Rounded corners (set to 0mm for square)
    boxrule=0.6pt,            % Border line width
    width=\textwidth,         % Use \textwidth for double-column spans
    left=8pt, right=8pt, top=8pt, bottom=8pt
]
\small
You are an expert data analysis assistant.
The text you are classifying is part of a dataset described as: \{dataset\_description\}.
We have a set of annotation attributes to classify for the text below by discovering appropriate categories within a \{max\_depth\}-level hierarchy.\par\medskip

(for each annotation attribute i = 1..n) Annotation \{i\}: "\{annotation\_name\}"\par
Description: \{annotation\_description\}\par
(if not root level) Parent Category Allocation Path for this Annotation: \{parent\_path\}\par\medskip

Hierarchy Context:\par
- Total Taxonomy Depth: \{max\_depth\} levels\par
- Current Classification Level: Level \{current\_level\} of \{max\_depth\}\par
(optional) - Hierarchy Examples (for granularity control only; you should propose category only based on the provided data):\par
\quad * \{annotation\_name\}: \{example\_path\} \par\medskip

GLOBAL RULES FOR OPEN CATEGORY DISCOVERY: \par
1. For each annotation, analyze the text and propose a concise, high-level, canonical category/subcategory name appropriate for Level \{current\_level\} of \{max\_depth\} that accurately captures the primary core topic of the text.\par
2. (root level) If the text does not contain information about the annotation, or if the annotation is not applicable or specified in the text, select "none". (e.g., If the attribute is "Format constraints" or "User constraints", you should output "none" instead of descriptive phrases like "no constraints" or "unconstrained" if the raw text does not specify any constraints).\par
2. (deeper levels) For subcategory classification under a parent category, if the text does not contain more details beyond the parent category, or if you think the parent category is enough for the record or not appropriate for further classification, select "none". Do NOT over-classify.\par
3. Category / Subcategory naming rules:\par
   \quad- Every subcategory must logically entail all parent categories in its parent category allocation path.\par
   \quad- DO NOT append parent category names, ancestor category names, or the annotation name. Just regard them as a known context. NEVER use any parent category in the path as a subcategory.\par
   \quad- Be concise and simple without sacrificing precise meaning (1-2 words in lowercase English are preferred (e.g., "refund", "late fee"); but you can use more words if they are really necessary to better describe the data ).\par
   \quad- Strictly ban compound category/subcategory styles (such as "a and b" type). Focus on only one main object or concept that fits the raw text best.\par
   \quad- Strip identifying details (proper nouns, personal names, numbers, specific dates, brand names); map to canonical conceptual category names.\par
4. Return the result ONLY as a valid JSON object mapping each annotation name to the proposed category name. For example:\par
  \quad\{\par
  \quad\quad  "annotation\_name\_1": "billing",\par
  \quad\quad  "annotation\_name\_2": "refund"\par
  \quad\}\par
5. Do not write explanations, outer markdown code block wrappers (like ```json), or any extra text.\par
6. The output must be in lowercase English.\par\medskip

Text to classify: \{text\}\par\medskip

Your JSON output (invalid JSON will be rejected):
\end{tcolorbox} 

\caption{The prompts used for proposing categories for multiple attributes.}
\label{fig:propose prompt}
\end{figure*}

\begin{figure*}[!t]
\centering
\begin{tcolorbox}[
    colback=gray!10,          % 10% gray background
    colframe=gray!40,         % Subtle border
    arc=2mm,                  % Rounded corners (set to 0mm for square)
    boxrule=0.6pt,            % Border line width
    width=\textwidth,         % Use \textwidth for double-column spans
    left=8pt, right=8pt, top=8pt, bottom=8pt
]
\small
You are an expert conversation analysis and data re-classification assistant.
The text you are re-classifying is part of a dataset described as: \{dataset\_description\}.
We have a set of annotation attributes and their candidate categories (which include confirmed high-frequency categories and initialized complementary categories).\par\medskip

(for each annotation attribute i = 1..n) Annotation \{i\}: "\{annotation\_name\}"\par
Description: \{annotation\_description\}\par
(if not root level) Parent Category Allocation Path for this Annotation: \{parent\_path\}\par
(root level) Candidate Categories:\par
(deeper levels) Candidate Subcategories under Parent:\par
   \quad- \{candidate\_1\}\par
   \quad- \{candidate\_2\}\par
   \quad- ...\par\medskip

Hierarchy Context:\par
- Total Taxonomy Depth: \{max\_depth\} levels\par
- Current Classification Level: Level \{current\_level\} of \{max\_depth\}\par\medskip

RE-CLASSIFICATION RULES: \par
1. (root level) For each annotation, analyze the text and select the single category from Candidate Categories that best fits the primary core meaning/topic of the text.\par
2. (root level) If none of the candidate categories fit the text, select "other".\par
3. (root level) If the text does not contain information about the annotation, or if the annotation is not applicable or specified in the text, select "none" (e.g., output "none" instead of "no constraints" or "unconstrained").\par
1. (deeper levels) For each annotation, analyze the text and select the single category from Candidate Subcategories that best fits the text within the parent category context.\par
2. (deeper levels) If none of the candidate subcategories fit the text, select "other".\par
3. (deeper levels) If the text does not contain more details beyond the parent category, or if the parent category is sufficient and not appropriate for further classification, select "none".\par
4. Return the result ONLY as a valid JSON object mapping each annotation name to the chosen category/subcategory name (or "other" / "none"). For example:\par
  \quad\{\par
  \quad\quad  "annotation\_name\_1": "billing",\par
  \quad\quad  "annotation\_name\_2": "other"\par
  \quad\}\par
5. Do not write explanations, outer markdown wrappers (like ```json), or any extra text.\par
6. The output must be in lowercase English.\par\medskip

Text to re-classify: \{text\}\par\medskip

Your JSON output (invalid JSON will be rejected):
\end{tcolorbox} 

\caption{The prompts used for re-classifying text based on complementary categories.}
\label{fig:reclassify prompt}
\end{figure*}

\begin{figure*}[!t]
\centering
\begin{tcolorbox}[
    colback=gray!10,          % 10% gray background
    colframe=gray!40,         % Subtle border
    arc=2mm,                  % Rounded corners (set to 0mm for square)
    boxrule=0.6pt,            % Border line width
    width=\textwidth,         % Use \textwidth for double-column spans
    left=8pt, right=8pt, top=8pt, bottom=8pt
]
\small
You are analyzing a dataset described as: \{dataset\_description\}.
Your task is to generate a COMPLEMENTARY taxonomy for the annotation attribute named '\{annotation\_name\}' (description: \{description\}).\par
(root level) Generate a FLAT (single-level, depth=1) set of top-level categories.\par
(deeper levels) Generate specific sub-categories to expand the parent category '\{parent\_path\}'.\par\medskip

(if any passed categories at this node) ALREADY ESTABLISHED / PASSED CATEGORIES (root level) or SUBCATEGORIES UNDER PARENT '\{parent\_path\}' (deeper levels):\par
The following categories have already been discovered and passed the filter:\par
   \quad- \{passed\_category\_1\}\par
   \quad- \{passed\_category\_2\}\par
   \quad- ...\par\medskip

RULES for complementary category generation:\par
   \quad- LOGICAL ENTAILMENT (deeper levels): Every sub-category must logically entail the parent category '\{parent\_path\}'. NEVER include any parent or ancestor category in the path as a sub-category.\par
   \quad- MUTUAL EXCLUSIVITY WITH ESTABLISHED CATEGORIES: The new complementary categories must be strictly MUTUALLY EXCLUSIVE with ALL the already established/passed categories listed above. Do NOT generate any categories that overlap, duplicate, or rephrase existing categories.\par
   \quad- MUTUAL EXCLUSIVITY AMONG NEW CATEGORIES: All generated complementary categories must be strictly MUTUALLY EXCLUSIVE with each other.\par
   \quad- COMPREHENSIVE COVERAGE: The complementary categories should cover the remaining meaningful aspects/topics of the annotation (root level) in this dataset / (deeper levels) under '\{parent\_path\}' that are NOT covered by the established categories.\par
   \quad- (root level) Do NOT generate sub-categories or nested levels. Only produce top-level category names.\par
   \quad- Strictly ban generic placeholder categories (do NOT generate "none", "other", "unspecified", "not applicable", "miscellaneous", "n/a").\par
   \quad- Strictly ban compound categories (such as "a and b").\par
   \quad- Be concise (1-3 words per category name in lowercase English).
\end{tcolorbox} 

\caption{The prompts used for generating complementary categories.}
\label{fig:comp tax prompt}
\end{figure*}

\section{Human Verification of Evaluation Agent}

We present an example execution trace of our offline QA agent in \Cref{tab:agent_example}. Using only \system's released synthetic table, the agent identifies the relevant reference population, inspects the available categories, and aggregates list-format requests across taxonomy levels without double counting. It then computes their proportion and compares it with the query's threshold. The displayed code and intermediate outputs allow human verification of the category mapping, denominator selection, numerical calculation, and final answer.

%==================================================

\section{Privacy Proof of \system}
\label{sec:appendix proof}

In this section, we present the complete mathematical proof for the privacy guarantee of \system (\Cref{alg:ship}). We consider the standard add-or-remove one DP in our work. We first provide a theorem about the relationship between DP and zCDP \citep{bun2016concentrateddifferentialprivacysimplifications}. 

\begin{theorem}[Approximate zCDP to DP]
\label{thm:zcdp_to_dp}
If a randomized mechanism $\mathcal{M}$ satisfies $\delta'$-approximate $\rho$-zCDP, then for any $\delta>\delta'$, $\mathcal{M}$ satisfies $(\varepsilon,\delta)$-DP, where
\begin{equation}
\begin{aligned}
\varepsilon = \inf_{\alpha>1}\Bigg\{ \alpha\rho&+\log\left(1-\frac{1}{\alpha}\right) \\
& -\frac{\log\left(\frac{\delta-\delta'}{1-\delta'}\right)+\log\alpha}{\alpha-1} \Bigg\}.
\end{aligned}
\end{equation}
\end{theorem}

Throughout the proof, the workload $\mathcal{W}$, dataset description $S_D$, maximum depth $d$, and models are fixed independently of the private corpus. We condition on data-independent model randomness, coupled consistently for records shared by neighboring datasets; all bounds hold uniformly over this randomness. Given the preceding DP outputs, category proposal, reclassification, and allocation are performed record-wise. Record-level proposals and assignments remain internal and are not released.

\begin{lemma}[Sensitivity of Partition Histograms]
\label{lem:app_sensitivity}
For any attribute $A$ and depth level $l$, fix the preceding DP outputs. Index histogram bins by parent--category pairs $(c_p,c)$ over a common category universe. Under add-or-remove-one-record adjacency, the joint count histogram $\mathbf{n}_{A,l}=(n_{c_p,c})$ across all active parent categories has global sensitivities
\[
\Delta_1 \le 1, \qquad \Delta_2 \le 1.
\]
The same bounds hold for the histogram under any individual parent category $c_p$.
\end{lemma}

\begin{proof}
With the preceding DP outputs fixed, each record is routed to at most one active parent category $c_p$ and contributes to at most one category count under that parent. Adding or removing one record leaves the routing and classification of all shared records unchanged. Consequently, the histogram difference is either zero or a vector with one nonzero entry equal to $+1$ or $-1$, giving
\[
\begin{aligned}
&\|\mathbf{n}_{A,l}(\mathcal{D})-\mathbf{n}_{A,l}(\mathcal{D}')\|_1\le 1,\\
&\|\mathbf{n}_{A,l}(\mathcal{D})-\mathbf{n}_{A,l}(\mathcal{D}')\|_2\le 1.
\end{aligned}
\]
Records whose paths terminate at \texttt{other} or \texttt{none} contribute no further counts. The argument applies to both open category proposals and closed reclassification once their respective classification contexts are fixed.
\end{proof}

\begin{lemma}[Privacy of Single-Level Hybrid Taxonomy Building]
\label{lem:app_single_level}
For any attribute $A\in\mathcal{A}$ and depth level $l\in\{1,\dots,d\}$, condition on the preceding DP outputs. With $\rho_{\text{open}}=0.9\rho_{\text{step}}$ and $\rho_{\text{close}}=0.1\rho_{\text{step}}$, the thresholded open outputs and closed noisy histograms generated by \Cref{alg:hybrid discovery} across all active parent categories at this level jointly satisfy $\delta_{\text{step}}$-approximate $\rho_{\text{step}}$-zCDP. Internal record assignments are excluded from this guarantee.
\end{lemma}

\begin{proof}
We analyze the open and closed stages jointly across all active parent categories.

In the open stage, histogram bins are indexed by parent--category pairs. For neighboring datasets $\mathcal{D}'=\mathcal{D}\cup\{x\}$, the added record can introduce at most one previously unobserved pair, whose count is one. Let $E'$ be the event that no such pair is released on $\mathcal{D}'$, and let $E$ be the certain event on $\mathcal{D}$. With $\sigma_{\text{open}}=1/\sqrt{2\rho_{\text{open}}}$ and $\tau_{\text{open}}=1+\sigma_{\text{open}}\Phi^{-1}(1-\delta_{\text{step}})$, an introduced pair survives with probability
\[
\Pr[1+Z\ge\tau_{\text{open}}]=\delta_{\text{step}},
\quad Z\sim\mathcal{N}(0,\sigma_{\text{open}}^2).
\]
Thus $\Pr[E']\ge 1-\delta_{\text{step}}$. Conditioning on $E'$ leaves the independent Gaussian noises on the common support unchanged. By \Cref{lem:app_sensitivity}, the conditioned output distributions $P$ and $Q$ are post-processing of Gaussian measurements with sensitivity at most one, so, for every $\alpha>1$,
\[
\max\{D_\alpha(P\Vert Q),D_\alpha(Q\Vert P)\}
\le \frac{\alpha}{2\sigma_{\text{open}}^2}
= \alpha\rho_{\text{open}}.
\]
Therefore, open discovery satisfies $\delta_{\text{step}}$-approximate $\rho_{\text{open}}$-zCDP. This is the Gaussian partition-selection guarantee of \citet[Theorem~13]{swanberg2023dpsipssimplerscalablemechanism}, specialized to one category per record.

Complementary proposal and semantic deduplication depend only on the surviving DP categories and public information, and therefore incur no additional privacy cost. Conditioned on these outputs, each parent has a fixed candidate domain $\mathcal{C}_{\text{open}}\cup\mathcal{C}_{\text{comp}}$. Reclassification remains record-wise, so the joint closed histogram has $L_2$ sensitivity at most one by \Cref{lem:app_sensitivity}. Measuring every candidate count, including zero counts, with independent Gaussian noise of scale $\sigma_{\text{close}}=1/\sqrt{2\rho_{\text{close}}}$ satisfies $\rho_{\text{close}}$-zCDP. Thresholding at $\tau_{\text{close}}$ and updating $\Omega_A$ are post-processing of the DP outputs.

By adaptive composition of approximate zCDP \citep[Lemma~8.2]{bun2016concentrateddifferentialprivacysimplifications}, the joint outputs satisfy $\delta_{\text{step}}$-approximate $(\rho_{\text{open}}+\rho_{\text{close}})$-zCDP, where
\[
\rho_{\text{open}}+\rho_{\text{close}}=\rho_{\text{step}}.
\]
This bound covers all active parents together: adding or removing one record affects at most one parent--category pair, so neither privacy parameter scales with the number of parents.
\end{proof}

\begin{proof}[Proof of \Cref{thm:ship_privacy}]
For the analysis, let $H$ denote the transcript containing all thresholded open-discovery outputs, closed Gaussian measurements, and their post-processing during taxonomy construction. Unfiltered record-level proposals and internal assignments are excluded from $H$.

By \Cref{lem:app_single_level} and adaptive composition across all attributes and hierarchy levels, $H$ satisfies $\delta_{\text{tax}}$-approximate $\rho_{\text{tax}}$-zCDP, using the additive composition bounds
\[
\sum_{A\in\mathcal{A}}\sum_{l=1}^{d}\rho_{\text{step}}=\rho_{\text{tax}},
\quad
\sum_{A\in\mathcal{A}}\sum_{l=1}^{d}\delta_{\text{step}}=\delta_{\text{tax}}.
\]
Early termination of record paths does not increase these bounds. The released taxonomy $\mathcal{T}=\{A:\Omega_A\}_{A\in\mathcal{A}}$ is post-processing of $H$.

For any fixed transcript $H$, each record's joint annotation $(y_i^{(A)})_{A\in\mathcal{A}}$ depends only on that record, $H$, and the fixed public inputs and model randomness. Adding or removing one textual record therefore adds or removes one row of the internal annotated table $\{(y_i^{(A)})_{A\in\mathcal{A}}\}_{i=1}^{N}$, leaving all shared rows unchanged. The privacy guarantee of AIM \citep{mckenna2024aimadaptiveiterativemechanism} consequently applies under the same adjacency relation, giving $\rho_{\text{aim}}$-zCDP for its synthesis output conditional on $H$.

Adaptive composition with taxonomy construction yields $\delta_{\text{tax}}$-approximate $\rho_{\text{total}}$-zCDP, where
\[
\rho_{\text{total}}=\rho_{\text{tax}}+\rho_{\text{aim}}.
\]
The released pair $(\mathcal{T},\widetilde{\mathcal{D}})$ is post-processing of the analyzed transcript and AIM output. For any target $\delta_{\text{tax}}<\delta<1$, define $\bar{\delta}=(\delta-\delta_{\text{tax}})/(1-\delta_{\text{tax}})$. Applying \Cref{thm:zcdp_to_dp} gives $(\varepsilon,\delta)$-DP with
\begin{equation*}
\begin{aligned}
\varepsilon=\inf_{\alpha>1}\Bigg\{
\alpha\rho_{\text{total}}&+\log\left(1-\frac{1}{\alpha}\right)\\
&-\frac{\log\bar{\delta}+\log\alpha}{\alpha-1}
\Bigg\}.
\end{aligned}
\end{equation*}
This establishes the guarantee in \Cref{thm:ship_privacy}.
\end{proof}

\section{Computational Cost Analysis}
\label{app:cost}

In this section, we provide an analysis of the computational cost of different insight extraction methods, focusing primarily on LLM API token consumption. Let $N$ denote the number of records, $L$ the average number of tokens per record, $L_p$ the average prompt overhead for a record-level LLM call, and $d$ the hierarchy depth. Since multiple attributes are processed jointly in our implementation, we do not introduce an additional multiplicative factor for the number of attributes. We use $T_{\mathrm{tax}}$ to denote the token cost of taxonomy-level LLM operations. These estimates characterize the dominant scaling behavior rather than exact API billing, since prompt and output lengths vary across datasets and intermediate taxonomies.

\noindent \textbf{\system.} At each hierarchy level, \system first performs open-set category discovery over the records assigned to all active nodes. Since these nodes partition the dataset, this requires one complete pass over the $N$ records. Records whose discovered categories do not survive partition selection are subsequently re-classified over the supplemented candidate domain, requiring at most one additional pass. Therefore, under the fixed-depth implementation used in our experiments,
$$
\begin{aligned}
T_{\system} &\geq dN(L+L_p) + T_{\mathrm{tax}}, \\
T_{\system} &\leq 2dN(L+L_p) + T_{\mathrm{tax}}.
\end{aligned}
$$
The lower bound corresponds to the case where no re-classification is needed, while the upper bound assumes that all records require a second pass at every level. In practice, the cost lies between these two extremes.

\noindent \textbf{Flat-Open / Flat-Close.}
Both flat structured baselines require only a single record-level classification pass. Although multiple attributes are extracted, they are processed jointly within the same LLM call in our implementation. Their dominant token cost is therefore
$$
\begin{aligned}
T_{\text{Flat-Close}} &\approx N(L+L_p) + T^*_{\text{tax}}\\
T_{\text{Flat-Opne}} &\approx N(L+L_p).
\end{aligned}
$$
For Flat-Close, $T^*_{\mathrm{tax}}$ mainly corresponds to constructing the predefined taxonomy, whereas for Flat-Open it includes category consolidation after open-set discovery. Subsequent DP aggregation or tabular synthesis does not incur additional LLM token cost.

\noindent \textbf{Urania.}
Urania performs a single LLM-based facet and keyword extraction pass over the corpus, followed by embedding computation and hierarchical clustering. Its record-level LLM token cost is therefore approximately
$$
T_{\text{Urania}} \approx N(L+L_p) + T_{\text{sum}},
$$
where $T_{\mathrm{sum}}$ denotes the cost of generating cluster summaries across the hierarchy. Unlike \system, Urania does not repeatedly process the raw records at each hierarchy level: embeddings are computed once, the records are hierarchically clustered for $d$ levels, and summaries are generated bottom-up from the resulting clusters. If $K_\ell$ denotes the number of clusters at level $\ell$ and $L_{\mathrm{sum}}^{(\ell)}$ the average token cost of summarizing one cluster, then
$$
T_{\text{sum}} \approx \sum_{\ell=1}^{d} K_\ell L_{\text{sum}}^{(\ell)}.
$$
In addition to its LLM token cost, Urania incurs non-LLM computation for embedding inference and hierarchical clustering. We report this cost separately rather than converting it into API-token equivalents, since it depends on the embedding model, hardware, and clustering implementation.

\noindent \textbf{Discussion.}
Overall, \system incurs higher LLM token cost than the flat structured baselines and more record-level API processing than Urania, with its cost scaling linearly with the hierarchy depth $d$. This additional computation reflects a trade-off between extraction cost and analytical utility: \system repeatedly processes the corpus to construct finer-grained hierarchical domains and recover information not captured by a single flat pass. As shown in our query-answering evaluation, this richer representation is associated with substantially stronger utility on both targeted and open-ended analytical queries. We therefore view the increased computational cost as an important limitation and a practical trade-off of \system, rather than a cost shared by all insight extraction approaches.

\begin{figure*}[!t]
\centering
\begin{tcolorbox}[
    colback=gray!10,          % 10% gray background
    colframe=gray!40,         % Subtle border
    arc=2mm,                  % Rounded corners (set to 0mm for square)
    boxrule=0.6pt,            % Border line width
    width=\textwidth,         % Use \textwidth for double-column spans
    left=8pt, right=8pt, top=8pt, bottom=8pt
]
\small
You are an expert data taxonomy cleanup assistant.
(root level) We have a set of newly discovered candidate categories for the annotation attribute "\{annotation\_name\}" (Description: \{description\}):\par
(deeper levels) We have a set of newly discovered subcategories under the parent category "\{parent\_path\}" for the annotation attribute "\{annotation\_name\}" (Description: \{description\}):\par\medskip

(optional, variant A) Granularity Reference Example: \{hierarchy\_example\}\par
(The reference example demonstrates the expected level of granularity/abstraction of canonical categories at each level. You must align the granularity of the merged canonical categories with this reference example.)\par
(optional, variant B) Reference Categories (granularity reference examples):\par
   \quad- \{reference\_category\_1\}\par
   \quad- \{reference\_category\_2\}\par
   \quad- ...\par
(The reference examples demonstrate the expected level of granularity/abstraction of canonical categories. You must align the granularity of the merged canonical categories with these reference examples.)\par\medskip

Newly Discovered Categories to Deduplicate and Merge:\par
   \quad- \{discovered\_category\_1\}\par
   \quad- \{discovered\_category\_2\}\par
   \quad- ...\par\medskip

Your task:\par
Analyze the newly discovered candidate categories, deduplicate them, and merge categories so that the resulting canonical categories are strictly MUTUALLY EXCLUSIVE, well-consolidated, and match the target granularity demonstrated by the reference examples.\par
For any cluster of semantically identical categories, near-synonyms, phrasing variants, or duplicate terms describing the same concept, choose or create a single canonical, concise category name that captures their shared core concept, and map all input categories in that cluster to that canonical name.\par\medskip

RULES:\par
1. Every input category in the "Newly Discovered Categories" list MUST be mapped to a canonical category.\par
2. (root level) Canonical category names should be clear, concise (1-4 words in lowercase).\par
2. (deeper levels) Canonical category names should be clear, concise (1-4 words in lowercase). Each merged category must strictly remain a valid sub-category of the parent category "\{parent\_path\}".\par
3. STRICT MUTUAL EXCLUSIVITY (DEDUPLICATE SYNONYMS \& DUPLICATE PHRASINGS):\par
   \quad- All canonical categories after merging MUST be strictly MUTUALLY EXCLUSIVE with each other. No two canonical categories can have duplicate or overlapping meaning at the same level.\par
   \quad- If multiple discovered categories are synonyms, near-synonyms, or slight phrasing variants of the same concept, you MUST merge them into a single canonical category.\par
4. GRANULARITY PRESERVATION (MERGING CANNOT CHANGE THE GRANULARITY):\par
   \quad- The reference examples demonstrate the expected level of granularity and abstraction for canonical categories.\par
   \quad- You MUST strictly preserve and align the granularity of the merged canonical categories with the reference examples.\par
   \quad- CRITICAL REQUIREMENT: MERGING CANNOT CHANGE THE GRANULARITY. Merging is strictly for removing redundant duplicates, near-synonyms, phrasing variations, and identical concepts at the SAME semantic level of abstraction. You MUST NOT collapse distinct sub-disciplines, distinct subfields, or valid sibling categories (for example, merging "web development", "database systems", "software testing", or "algorithms" into an overly broad umbrella category like "software engineering" is STRICTLY PROHIBITED because it destroys the category granularity).\par
   \quad- Valid, distinct sibling concepts must remain separate distinct canonical categories.\par
5. CONCISE CONSOLIDATION:\par
   \quad- The resulting set of canonical categories should be clean and well-consolidated without redundant duplicates.\par
   \quad- Consolidate minor wording variations or duplicate phrasings into clean canonical names, but never over-merge distinct sibling sub-topics.\par
6. "none" must map to "none" and not be merged with specific domain concepts.\par
7. All canonical category names must be in lowercase English.\par
8. Return the result ONLY as a valid JSON object where keys are the original input categories and values are the corresponding canonical/merged categories. For example:\par
  \quad\{\par
  \quad\quad  "discovered\_cat\_1": "discovered\_cat\_1",\par
  \quad\quad  "discovered\_cat\_2": "merged\_canonical\_name",\par
  \quad\quad  "discovered\_cat\_3": "merged\_canonical\_name"\par
  \quad\}\par
9. Do not write explanations, code block markers (like ```json), or any extra text.\par\medskip

Your JSON output (invalid JSON will be rejected):
\end{tcolorbox} 

\caption{The prompts used for de-duplicating proposed and survived categories.}
\label{fig:merge prompt}
\end{figure*}

% \begin{table}[!t]
% \centering
% \footnotesize
% \caption{Average category counts and allocated records across attributes on WildChat ($N=16{,}703$).}
% \label{tab:wildchat_category_breakdown_compact}
% \renewcommand{\arraystretch}{1.15}
% \setlength{\tabcolsep}{4.5pt}
% \begin{tabular}{cc ccc}
% \toprule
% \textbf{$\varepsilon$} & \textbf{Level} & \textbf{Open-Set} & \textbf{Close-Set} & \textbf{Other \& None} \\
% \midrule
% \multirow{3}{*}{$1.0$} 
%   & L0 & 13.5 (83.9\%) & 4.5 (2.5\%) & 1.8 (13.7\%) \\
%   & L1 & 35.8 (80.2\%) & 43.0 (8.3\%) & 23.0 (11.5\%) \\
%   & L2 & 29.8 (46.7\%) & 152.0 (37.6\%) & 74.0 (15.7\%) \\
% \midrule
% \multirow{3}{*}{$5.0$} 
%   & L0 & 16.3 (86.5\%) & 4.3 (1.0\%) & 1.8 (12.5\%) \\
%   & L1 & 52.8 (89.3\%) & 37.3 (5.8\%) & 24.8 (5.0\%) \\
%   & L2 & 55.0 (55.5\%) & 137.5 (32.1\%) & 93.0 (12.3\%) \\
% \midrule
% \multirow{3}{*}{$10.0$} 
%   & L0 & 20.8 (86.3\%) & 2.8 (0.6\%) & 2.3 (13.1\%) \\
%   & L1 & 86.3 (86.1\%) & 30.3 (4.6\%) & 29.5 (9.4\%) \\
%   & L2 & 70.8 (51.7\%) & 147.3 (33.0\%) & 119.8 (15.4\%) \\
% \bottomrule
% \end{tabular}
% \end{table}

\begin{figure*}[!t]
\centering
\begin{tcolorbox}[
    colback=gray!10,          % 10% gray background
    colframe=gray!40,         % Subtle border
    arc=2mm,                  % Rounded corners (set to 0mm for square)
    boxrule=0.6pt,            % Border line width
    width=\textwidth,         % Use \textwidth for double-column spans
    left=8pt, right=8pt, top=8pt, bottom=8pt
]
\small
(\textbf{Yes/No}) SCORING RULE:\par
- Score 5 (Match): The candidate answer gives the same binary verdict as Ground Truth (both are "Yes" or both are "No").\par
- Score 1 (Mismatch): The candidate answer contradicts Ground Truth (one says "Yes" while the other says "No"), hedges ambiguously, or fails to provide a clear verdict. \par\medskip 

(\textbf{Comparison}) SCORING RUBRIC (1 to 5):\par
- [5] Perfect / Correct Winning Concept: Candidate correctly identifies the identical winning concept/category as Ground Truth. Natural language phrasing variations, synonyms, and expanded descriptive forms (e.g. "text transformation" vs "transforming and modifying existing text") are fully credited.\par
- [4] Correct Winner with Minor Imprecision: Candidate correctly identifies the winning concept, but its comparative margin or secondary supporting numbers have slight discrepancies.\par
- [3] Ambiguous / Inconclusive Comparison: Candidate describes both concepts without declaring a clear winner, falsely asserts they are equal/tied, or provides vague comparative evidence.\par
- [2] Inverted Comparative Direction: Candidate incorrectly declares the losing concept as the winner (e.g. claiming entity B exceeds entity A when Ground Truth establishes A exceeds B).\par
- [1] Completely Incorrect / Irrelevant: Candidate identifies an entity completely outside the comparison, hallucinates findings, or failed to answer.\par\medskip  

(\textbf{Numerical}) SCORING RUBRIC (1 to 5):\par
- [5] Perfect / Highly Accurate: Candidate value has absolute error $\leq$ 0.05 (5 percentage points) OR relative error $\leq$ 10\% compared to Ground Truth. Correctly identified the target subset and denominator.\par
- [4] Minor Numerical Deviation: Candidate value has absolute error 0.05-0.10 OR relative error 10\%-20\%. Core calculation is sound with minor discrepancy due to sample noise or approximation.\par
- [3] Moderate Numerical Deviation: Candidate value has absolute error 0.10-0.25 OR relative error 20\%-50\%. Noticeable numerical discrepancy or partially imprecise subset filtering.\par
- [2] Severe Distortion: Candidate value has absolute error $>$ 0.25 OR relative error $>$ 50\%, or severely misidentified the denominator/reference population.\par
- [1] Completely Incorrect: Numerical value is completely wrong, off by orders of magnitude, unparseable, or missing. \par\medskip

(\textbf{Open-ended statements}) Rubric:\par
Focus primarily on semantic alignment; the numerical error bounds in the rubric below are provided for reference.\par
- [5] perfect alignment: The statement object is supported by a logically sound descriptive conclusion, and its statistical result (if given) maintains a strict error bound (EITHER a relative error of $<$ 10\% OR an absolute error of $<$ 5\%, depending on the statement's context).\par
- [4] minor deviation: The descriptive conclusion preserves the primary narrative takeaway, while the paired statistical result (if given) contains a slight error (EITHER a relative error between 10\% and 20\% OR an absolute error between 5\% and 10\%, chosen based on context).\par
- [3] middle deviation: The descriptive conclusion blends true and false information that slightly alters the narrative, alongside a noticeable statistical result (if given) showing EITHER a relative error between 20\% and 50\% OR an absolute error between 10\% and 25\%, evaluated based on the statement's context.\par
- [2] severe distortion: The statement object is related to the source data but is fundamentally misrepresented by a severely distorted descriptive conclusion, or the statistical result (if given) introduces a massive error (EITHER a relative error of $>$ 50\% OR an absolute error of $>$ 25\%, evaluated based on the statement's context).\par
- [1] complete absence / insufficient information: The statement object is entirely absent from or unrelated to the source data, OR the available data lacks sufficient information to validate it (e.g., missing critical context such as asserting 'undergraduate students prefer X' when data provides no information about student types as the condition, or core claims cannot be found such as asserting 'Python web development' when data only shows 'general programming').
\end{tcolorbox} 

\caption{The prompts used for scoring the query answers or the statements.}
\label{fig:score prompt}
\end{figure*}

\section{Prompts} 
\label{app:prompt}

This paper involves many LLM API calls during execution, which requires proper prompts. Here, we provide four prompts that are used in \system: 
\begin{enumerate}[leftmargin=*, label=\textbullet]
    \item The prompt used to propose categories for each record is provided in \Cref{fig:propose prompt}. These categories will go through partition selection after proposal.
    \item \Cref{fig:reclassify prompt} provides the prompt used to re-classify the text given generated complementary categories and proposed categories that pass the partition selection.
    \item The prompt used to generate complementary categories is provided in \Cref{fig:comp tax prompt}. 
    \item The prompt used to merge duplicate categories is provided in \Cref{fig:merge prompt}. This step is conducted on those categories passing partition selection.
\end{enumerate}
Moreover, we provide the prompts used in QA scoring in \Cref{fig:score prompt}.

\section{Generative AI Usage}

We mainly use the Gemini family for coding and polishing writing, including main sections and appendix. The authors reviewed and verified the AI-assisted material and take full responsibility for the work.

\end{document}